\documentclass[manuscript]{acmart}
\setcopyright{acmlicensed}
\copyrightyear{2018}
\acmYear{2018}
\acmDOI{XXXXXXX.XXXXXXX}
\acmConference[Conference acronym 'XX]{Make sure to enter the correct
  conference title from your rights confirmation email}{June 03--05,
  2018}{Woodstock, NY}
\acmISBN{978-1-4503-XXXX-X/2018/06}
\usepackage{listings}%
\usepackage{amsmath,amsfonts}%
\usepackage{mathrsfs}%
\usepackage{xcolor}%
\usepackage{manyfoot}%
\usepackage{makecell}
\usepackage{multirow}
\usepackage[utf8]{inputenc} 
\usepackage{graphicx}
\usepackage{threeparttable}
\usepackage{hyperref}       
\usepackage{url}            
\usepackage{booktabs}       
\usepackage{nicefrac}       
\usepackage{microtype}      
\usepackage[table]{xcolor}         

\usepackage{soul}

\usepackage{graphicx}
\usepackage{bbm}
\usepackage{balance}

\usepackage{caption}
\usepackage{subcaption}
\usepackage{textcomp}
\usepackage{amsmath}

\usepackage{wrapfig}
\usepackage[capitalize,noabbrev]{cleveref}
\usepackage{lscape}
\usepackage[normalem]{ulem}

\usepackage{algorithmic}
\usepackage{algorithm}    

\newcommand{\model}{\textsc{HRLFS}}

\newtheorem{assumption}{Assumption}

\begin{document}

\title[Comprehend, Divide, and Conquer]{Comprehend, Divide, and Conquer: Feature Subspace Exploration via Multi-Agent Hierarchical Reinforcement Learning}


\author{Weiliang Zhang}
\affiliation{%
  \institution{Computer Network Information Center, the Chinese Academy of Sciences}
  \city{Beijing}
  \country{China}
  }
\affiliation{%
  \institution{University of the Chinese Academy of Sciences}
  \city{Beijing}
  \country{China}
}
\email{wlzhang@cnic.cn}

\author{Xiaohan Huang}
\affiliation{%
  \institution{Computer Network Information Center, the Chinese Academy of Sciences}
  \city{Beijing}
  \country{China}
  }
\affiliation{%
  \institution{University of the Chinese Academy of Sciences}
  \city{Beijing}
  \country{China}
}
\email{xhhuang@cnic.cn}

\author{Ziyue Qiao}
\affiliation{%
  \institution{School of Computing and Information Technology, Great Bay University}
  \city{Dongguan}
  \country{China}
}
\email{zyqiao@gbu.edu.cn}

\author{Qingqing Long}
\affiliation{%
  \institution{Computer Network Information Center, Chinese Academy of Sciences}
  \city{Beijing}
  \country{China}
}
\affiliation{%
  \institution{University of the Chinese Academy of Sciences}
  \city{Beijing}
  \country{China}
}
\email{qqlong@cnic.cn}

\author{Zhen Meng}
\affiliation{%
  \institution{Computer Network Information Center, Chinese Academy of Sciences}
  \city{Beijing}
  \country{China}
}
\email{zhenm99@cnic.cn}

\author{Yuanchun Zhou}
\affiliation{%
  \institution{Computer Network Information Center, Chinese Academy of Sciences}
  \city{Beijing}
  \country{China}
  }
\affiliation{%
    \institution{University of the Chinese Academy of Sciences}
  \city{Beijing}
  \country{China}
}
\email{zyc@cnic.cn}

\author{Yi Du}
\affiliation{%
  \institution{Computer Network Information Center, Chinese Academy of Sciences}
  \city{Beijing}
  \country{China}
}
\email{duyi@cnic.cn}

\author{Meng Xiao}
\authornote{Corresponding author: Meng Xiao (shaow@cnic.cn).}
\affiliation{%
  \institution{Computer Network Information Center, Chinese Academy of Sciences}
  \city{Beijing}
  \country{China}
}
\affiliation{%
  \institution{Duke-NUS Medical School, National University of Singapore}
  \city{Singapore}
  \country{Singapore}
}
\email{shaow@cnic.cn}
\renewcommand{\shortauthors}{Zhang et al.}


\begin{abstract}
Feature selection aims to preprocess the target dataset, find an optimal and most streamlined feature subset, and enhance the downstream machine learning task. 
Among filter, wrapper, and embedded-based approaches, the reinforcement learning (RL)-based subspace exploration strategy provides a novel objective optimization-directed perspective and promising performance. 
Nevertheless, even with improved performance, current reinforcement learning approaches face challenges similar to conventional methods when dealing with complex datasets. 
These challenges stem from the inefficient paradigm of using one agent per feature and the inherent complexities present in the datasets.
This observation motivates us to investigate and address the above issue and propose a novel approach, namely \model{}. 
Our methodology initially employs a Large Language Model (LLM)-based hybrid state extractor to capture each feature's mathematical and semantic characteristics. 
Based on this information, features are clustered, facilitating the construction of hierarchical agents for each cluster and sub-cluster. 
Extensive experiments demonstrate the efficiency, scalability, and robustness of our approach. 
Compared to contemporary or the one-feature-one-agent RL-based approaches, \model{} improves the downstream ML performance with iterative feature subspace exploration while accelerating total run time by reducing the number of agents involved. \footnote{Our code and dataset are publicly accessible toward ~\href{https://github.com/coco11563/HARLFS}{Github} and ~\href{https://huggingface.co/datasets/Shaow/Feature_Selection_Dataset}{Huggingface}.}
\end{abstract}


  
\maketitle

\section{Introduction}

Feature selection (FS) plays a critical role in classical machine learning models by mitigating the curse of dimensionality, reducing training times, and alleviating feature redundancy, thereby substantially enhancing predictive performance~\cite{li2017feature, wang2025towards, li2017recent}. 
The evolution of FS methodologies has progressed from early heuristic and statistical screening techniques to the development of sophisticated approaches that incorporate filter-based~\cite{biesiada2008feature,yu2003feature,song2021fast}, wrapper-based~\cite{stein2005decision,gheyas2010feature}, and embedded strategies~\cite{dinh2020consistent,li2016deep,lemhadri2021lassonet} tailored to manage increasingly complex data environments. Concurrently, the rapid advancement of modern artificial intelligence~\cite{qin2025scihorizon,huang2026scihorizon,liugut, xiao2026bioharness} has broadened the application domain of FS to encompass challenging tasks such as data pattern discovery~\cite{xiao2023beyond}, biomarker identification~\cite{11164312,ying2024revolutionizing}, and the construction of AI-driven data pipelines (AI4data)~\cite{liang2022advances,cai2025knowledge,huang2025collaborative,xiao2025m}. 
In the context of emerging demands in multi-modal data processing and real-time big data analytics, integrating efficient and intelligent feature selection is increasingly recognized as a pivotal element in enhancing data quality and advancing data-centric paradigms~\cite{fu2024tabular}.

\begin{figure}[!h]
    \centering
\includegraphics[width=0.65\linewidth]{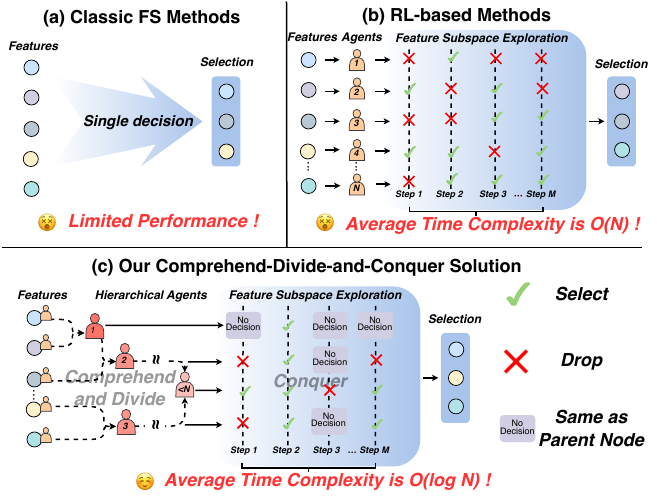}
    \caption{\textcolor{black}{Comparison between \model\ with other feature selection approaches. }}
    \label{fig_introduction}
\end{figure}

Among the various feature selection methodologies, reinforcement learning (RL)-based strategies~\cite{sarlfs} have received significant attention due to their ability to optimize feature subsets in an objective-directed manner with iterative feature subspace exploration. 
Despite these advances, RL-based feature selection methods encounter notable challenges when handling complex datasets. 
This difficulty primarily stems from the limited capability of a single-agent-all-feature framework~\cite{sarlfs}. 
Current RL-based research has made some progress in addressing the complexity of datasets. 
An insightful approach~\cite{marlfs} involves adopting a multi-agent reinforcement learning~\cite{Zhang2021} (MARL) architecture for feature sub-space exploration. 
However, the limitation of the one-agent-per-feature system becomes evident, as it necessitates an excessive number of agents when processing high-dimensional datasets. 

In summary, as depicted in Figure~\ref{fig_introduction}(a), classic FS methods perform feature selection through a single decision, which is efficient but analyzes the nature of the features inadequately, resulting in lower performance. 
Present RL-based Methods as depicted in Figure~\ref{fig_introduction}(b). 
Those methods of subspace exploration come from wrapper feature selection ideas and have powerful feature selection capabilities. 
Still, they also have the inherent drawbacks of wrapper methods, i.e., they have $\mathcal{O}(N)$ time complexity and require a lot of time when dealing with large datasets. 
To mitigate the challenge of managing an overwhelming number of agents, group-based~\cite{fan2021autogfs} and interaction-wise~\cite{fan2020autofs} agent architectures have been proposed as promising solutions.
Nevertheless, these approaches rely solely on superficial mathematical characteristics (e.g., standard variance of features) to organize features, resulting in inaccuracies due to neglecting semantic and contextual feature relationships~\cite{li2024exploring,moraffah2024causal}. 

Those observations motivate us to develop novel hierarchical reinforcement learning-based architectures for feature sub-space exploration.
The core concept of this research is that when a feature (e.g., age) is irrelevant to a task, its semantically equivalent feature (e.g., birthdate) may also be irrelevant. 
Additionally, two features exhibiting similar numerical patterns (e.g., highly related biological signals) could be redundant, necessitating fine-grained differentiation. 
For example, Mice-Protein~\cite{higuera2015self} is utilized in research on mouse Down syndrome and comprises 77 key proteins linked to learning ability in the mouse brain. 
It contains two genes, 'pGSK3B\_N' and 'GSK3B\_N', that are highly correlated biologically, because they correspond to the same protein, 'GSK-3$\beta$', which represents different forms of the protein during its expression~\cite{higuera2015self}.
Conventional feature selection methods struggle with such intrinsic biological redundancy. 
A one-agent-one-feature approach, for example, makes independent decisions and often fails to remove both features due to a lack of agent collaboration. 
Meanwhile, group-based and interaction-wise methods based on purely mathematical correlations can detect their statistical link but lack the semantic context to understand the underlying biological reason.
Our approach uses semantic information to identify latent relationships and employs a hierarchical structure for unified decision-making. This addresses the limitations of the aforementioned methods in terms of feature understanding, agent collaboration, and decision-making efficiency.
%

As depicted in Figure~\ref{fig_introduction}(c), our solution builds hierarchical agents to make decisions on feature subspace exploration tasks, reducing the average decision time complexity to $\mathcal{O}(log\ N)$ while maintaining high performance. 
The core idea can be divided into three stages: 
(\textbf{\underline{Comprehend}}) Our primary innovation lies in integrating Large Language Models to comprehend the semantic meanings of feature metadata, coupled with Gaussian Mixture Models (GMM) to capture the mathematical characteristics of the features.
(\textbf{\underline{Divide}}) Building upon these enhanced feature representations, we employ a clustering mechanism that groups similar features.
(\textbf{\underline{Conquer}})  Utilizing these clusters, we construct a multi-agent hierarchical reinforcement learning framework that mirrors the natural organization of the feature space. 
This hierarchical structure strategically divided decision-making responsibilities to cluster-specific agents and adaptively reduced the required activated agents, significantly reducing computational overhead and accelerating the exploration process.
Our contributions can be summarized as:

\begin{itemize}
    \item \textbf{Comprehensive Feature Understanding.} We harness Large Language Models and Statistical Differential to extract meaningful information from dataset metadata and their mathematical characteristics, thereby enabling a deeper understanding of feature relationships beyond numerical statistics. 
    \item \textbf{Beyond One-agent-one-feature Architecture.} The clustered features are managed through a hierarchical multi-agent reinforcement learning architecture, which reduces the total number of activated agents, enhances computational efficiency, and improves the feature selection process. 
    
    \item \textbf{Theoretical and Empirical Efficiency.} We rigorously demonstrate the efficiency and effectiveness of our method from both theoretical and experimental perspectives, showing its superior predictive performance and computational scalability compared to existing feature selection approaches.
\end{itemize}

\section{Important Definitions}
\noindent\textbf{Feature Selection.} Given the dataset $D = \{X\in \mathcal{R}^{m\times n},y\in \mathcal{R}^{m\times 1}, F\}$, where $X$ and $y$ are the features and labels, respectively. 
$m$ and $n$ is the number of samples and features.
We use a finite set $F=\{f_1,f_2,...,f_n\}$ to indicate the feature column included in dataset $D$, where $f_i$ is the $i-th$ column of $X$. 
The goal of feature selection is to find the optimal feature subset $F^*\subset F$  to enhance model performance while maintaining computational efficiency.

\smallskip
\noindent\textbf{Combine Feature Selection with RL.} 
In this paper, we frame the decision process of the feature selection method as a Markov Decision Process (MDP)~\cite{feinberg2012handbook}. 
Specifically, $s_t$ represents the state of the selected feature $F_t$ at time step $t$. 
The agent(s), with policy function(s), $\pi(\cdot)$, will select or drop each feature with an action $a_t\in \{0,1\}^n$, where each action component corresponds to including (1) or excluding (0) a specific feature. 
With the action $a_t$, we could build a subset of the feature $F_{t+1}$ and extract its related new state $s_{t+1}$. 
We can also evaluate the selection and obtain the reward $r_t$. 
With the collection of memory $m_t = (s_t, a_t, s_{t+1}, r_t)$, we could optimize the policy function(s) $\pi(\cdot)$ toward reinforcement learning and finally obtain the optimal selection $F^*$.

\section{Methodology}

\subsection{Overview of \model}
\model\ is a highly effective RL-based feature subspace exploration method that adopts a comprehend-divide-and-conquer paradigm. 
\model\ will first extract the hybrid state of each feature, cluster each feature, and then initialize the hierarchical agent architecture. 
By using the state of each feature, the hierarchical agents will explore the feature combination and optimize its selection policy. 
After that, \model\ will output the final optimal selection to enhance the downstream machine-learning task.

\subsection{Hybrid Feature State Extraction}\label{sec:feature_embedding}
As illustrated in Figure~\ref{fig_1}, we develop a hybrid-faceted feature state extraction method to help both the clustering component and each RL agent comprehend the given dataset. 
\begin{figure}[!h]
    \centering
\includegraphics[width=0.55\linewidth]{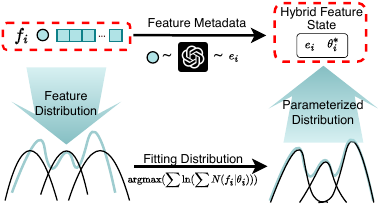}
    \caption{Hybrid feature state extraction.}
    \label{fig_1}
\end{figure}

\smallskip
\noindent\textbf{Leveraging Feature Distribution.}
We first consider the mathematical characteristics of each feature. 
To achieve that, we model each feature $f$ using a Gaussian Mixture Model with $k$ Gaussian component to accurately capture the distribution of the feature values. 
The probability density function (pdf) of the feature $f$ is defined as:
\begin{equation}
    P(f| \theta) = \sum^{k}_{i = 1} z_i N_i\left(f \mid \mu_i, \sigma_i\right) \text{, where} \sum^{k}_{i = 1} z_i = 1.
\end{equation}
Here, $N_i(f \mid \mu_i,\sigma_i)$ denote the $i$-th Gaussian distribution, $z_i$ is a non-negative constant that controls the weight of each component. $\theta=(z_1,...,z_k,\mu_1,...,\mu_k,\sigma_1,...,\sigma_k)$ are the overall parameters of the GMM model. 
We then use the following log-likelihood function to maximize the likelihood between PDF to real distribution of $F$, given as:
\begin{equation}
    l(\Theta \mid F) = 
    \sum_{f\in F} \ln\left(\sum^{k}_{i = 1} z_{i} N\left(f \mid \mu_{i}, \sigma_{i}\right)\right),
\end{equation}
where $\Theta=\{\theta_i\}_{i=1}^n$ denotes all the parameters for $n$ features. 
By that, we apply the expectation maximization (EM) algorithm to optimize the loss function and obtain the final parameter $\theta^*_i$ as the distribution state of the feature $f_i$.

\noindent\textbf{Leveraging Feature Metadata.} 
Another key source to address feature state inaccuracy is to obtain semantic information from feature descriptions (i.e., metadata). 
We employ a straightforward but efficient approach: input the feature name-description pair into a large language model and subsequently utilize PCA~\cite{dunteman1989principal} for dimensionality reduction to {align with the Gaussian component number $k$}.
This process results in sentence embeddings that serve as the semantic feature state.
For incomplete datasets, we use the following prompt to generate its feature description. 
As shown in the figure~\ref{llm_prompt}, we feed the description of the dataset, the available feature names and their descriptions, and the names of the features without descriptions to the large language model to complete the missing feature description.
\begin{figure*}[!h]
\centering 
\includegraphics[width=0.85\textwidth]{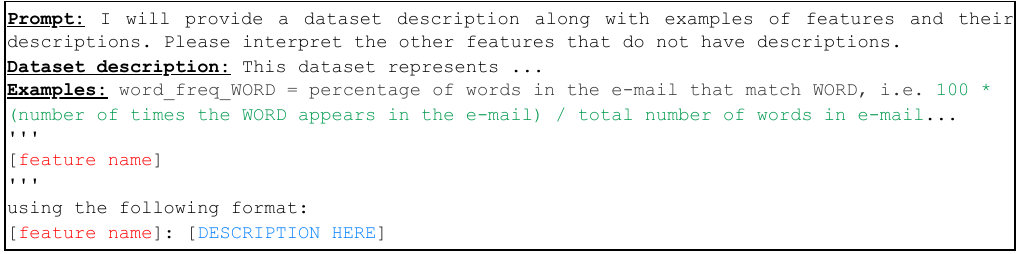}
\caption{The prompt to generate feature description using dataset metadata.}
\label{llm_prompt}
\end{figure*}
{Generated metadata are created once during preprocessing,
stored as immutable annotations.} 
Further, we use {an all-zero embedding} as its semantic feature state for the dataset with both no dataset description and feature description (e.g., a synthetic dataset). 
By that, we can obtain the semantic state {$e_i \in \mathbf{R}^{k}$} for each feature $f_i$.


Finally, we concatenate the distribution and semantic states to generate the hybrid state, given as $h_i = e_i \oplus \theta_i^*$.

\subsection{Divide-and-Conquer Agent Architecture} \label{sec:hierarchical_agent}

We then introduce how we build the agent hierarchy (divide) and organize the agent decision (conquer) to overcome the challenge of the dataset complexity.

\begin{figure}[!h]
    \centering
\includegraphics[width=0.55\linewidth]{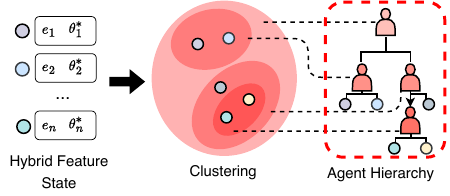}
    \caption{Construction of agent hierarchy by incremental clustering.}
    \label{fig_3}
\end{figure}

\subsubsection{Agent Hierarchy Construction.}\label{sec:hierarchical_agent_construction}

\textcolor{black}{As illustrated in Figure~\ref{fig_3}, we propose H-clustering algorithm that aggregates the features by the similarity of their hybrid state and uses this group information to build the hierarchy of the agents.}
Specifically, we employ an incremental clustering algorithm to merge two clusters that cause the smallest increase in variance~\cite{ward1963hierarchical}. 
Algorithm~\ref{algorithm:feature_aggregation} illustrates that starting with all feature states $H = \{h_i\}_{i=1}^n$, where every feature initially forms a cluster, and new clusters are sequentially merged through similarity grouping until just one remaining cluster is achieved. 
We then derive the cluster set $\mathcal{C}_{\text{agent}}$ where each cluster $C_i \in \mathcal{C}_{\text{agnet}}$ signifies an Agent $\mathcal{A}_i$. 
Inclusion relationships establish the agents' hierarchy; a lower-level agent $\mathcal{A}_i$ is part of a higher-level agent $\mathcal{A}_j$ if $C_i\subset C_j$  (with two lower-level agents combining to form a higher-level agent).

\renewcommand{\algorithmicrequire}{\textbf{Input:}}
\renewcommand{\algorithmicensure}{\textbf{Output:}}
\begin{algorithm}
    \caption{H-clustering: Agent Hierarchy Construction}\label{algorithm:feature_aggregation}
    \begin{algorithmic}[1]
        \REQUIRE{Feature number $n$, feature embeddings $H=\{h_1,...,h_n\}$}
        \ENSURE{A set of clusters $\mathcal{C}_{\text{agent}}$}
        \STATE {\color{gray}\# Initialize each feature as a cluster.}
        \STATE $C_i\gets\{f_i\}$
        \STATE $\mathcal{C}\gets\ \{C_1,...,C_n\}$
        \STATE $\mathcal{C}_{\text{agent}}\gets\ \{\}$
        \WHILE{ $|\mathcal{C}| > 1$} 
        \STATE {\color{gray}\# Calculate the distance between each two clusters by Ward's method.}
        \STATE $d_{ij}=\frac{|C_i| |C_j|}{|C_i| + |C_j|} \| \overline{h_{ci}} - \overline{h_{cj}} \|^2$, where $|\cdot|$ represents the number of items 
        
        and $\overline{h_{ci}},\overline{h_{cj}}$ is the mean vector of each feature's state within $C_i$ and $C_j$. 

        \STATE {\color{gray}\# Find the two most similar clusters. Then aggregate them}
        {\color{gray} into new cluster.}
        \STATE $C_{\text{new}}\gets C_a \cup C_b$
        \STATE {\color{gray}\# add $C_{new}$ to $\mathcal{C}$ and $\mathcal{C}_{\text{agent}}$.}
        \STATE $\mathcal{C} \gets\mathcal{C} \setminus C_a$
        \STATE $\mathcal{C} \gets\mathcal{C} \setminus C_b $
        \STATE $\mathcal{C} \gets\mathcal{C} \cup \{C_{\text{new}}\}$
        \STATE $\mathcal{C}_{\text{agent}} \gets\mathcal{C}_{\text{agent}} \cup \{C_{\text{new}}\}$
        \ENDWHILE
    \end{algorithmic}
\end{algorithm}   

\smallskip
\noindent\textbf{Hierarchical Agent for Cluster-specific Decision.}
As shown in Figure \ref{main_fig_2}, agents work together to make decisions across various granular levels within a hierarchical framework, extending from higher to lower-level agents. 
Each agent shares states and rewards while developing policies. 
The leaf nodes of the tree correspond to individual features and are responsible for making fine-grained decisions on whether to select or discard each specific feature. In contrast, the internal (non-leaf) nodes represent higher-level feature clusters and make broader decisions by partitioning the feature space, determining whether entire groups of features—represented by their respective subtrees—should be further considered or pruned from the selection process.

\smallskip
\noindent$\diamondsuit$ \textit{\uline{Action:}}
At the $t$-th iteration, the action $a^i_t$ associated with the $i$-th agent (corresponding to a node in the tree) is a binary choice: $a^i_t \in {\text{select}, \text{drop}}$.
If an agent (tree node) chooses the 'select' action, it recursively delegates the feature selection task to its child nodes (i.e., the subtrees rooted at its children are further explored). Conversely, if the 'drop' action is selected by a node, the entire subtree rooted at this node becomes inactive, and none of its descendant agents (or their associated feature clusters) are activated. In this way, a 'drop' action at any internal node results in a pruning of the corresponding subtree from the selection process. 

\smallskip
\noindent$\diamondsuit$ \textit{\uline{State}}
\label{section:agent_state} 
The state for each agent is a vectorized representation derived from the selected feature subset. 
We adopt the hybrid feature state as the state representation for each agent. 
We first acquire the parameter $\theta^*_i = \{z^{i}_1,\cdots,z^{i}_k,\mu^{i}_1,\cdots,\mu^{i}_k,\sigma^{i}_1,\\\cdots,\sigma^{i}_k\}$ from GMM for the feature $f_i$. 
Due to the $z$ adjustment of the weight of each Gaussian component, we weighted-sum each $\mu$ and $\sigma$ toward the corresponding $z$ to form a state vector $s_i$, formally:
\begin{equation}
    s_i =    e_i 
\oplus \sum_{j=1}^k z_j^i * (\mu_j^i \oplus \sigma_j^i),
\end{equation}
To obtain the state of step $t$, we concatenate all feature states to form a fixed-size vector {$\mathbf{s}_t \in \mathbf{R}^{n\times (k + 2)}$}:
\begin{equation}
    \mathbf{s}_t = \bigoplus_{i=1}^n\left(\mathbbm{1}_t(f_i) * s_i\right)
\end{equation}
where $\mathbbm{1}_t(\cdot)$ is an indicator function denoting that the feature $f_i$ is selected or not in step $t$. 
If $f_i$ is not chosen, the function yields 0; if selected, it yields 1.

\smallskip
\noindent$\diamondsuit$ \textit{\uline{Policy}}
The agent‘s policy network is a feed-forward neural network with a binary classification head. Formally, for feature $f_i$, its action in $t$-th iteration
is then derived by $a^i_{t}=\pi_i(\mathbf{s}_{t-1})$.

\smallskip
\noindent$\diamondsuit$ \textit{\uline{Reward}}
\label{section:reward}
As illustrated in Figure~\ref{main_fig_2}, we design the reward function regarding downstream task performance and quantity suppression of the selected feature numbers.

$\star$ \textit{Performance}: the first aspect of the reward function evaluates performance through downstream tasks, such as classification and regression. 
We train a downstream task model using a selected feature subset and define $r^p_t$ using the model evaluation metrics.

$\star$ \textit{Quantity Suppression}: The second perspective focuses on ensuring a compact number of features through:
\begin{equation} \label{equation:reward_lambda}
    r^q_t=\frac{|F|-|F_t|}{|F|+\lambda \cdot |F_t|}
\end{equation}
where $F$ is the entire feature set, $F_t$ is the selected feature subset in step-$t$, $\lambda$ is a hyperparameter, and $|\cdot|$ denoted the size of given set.
As $\lambda$ increases, the penalty for keeping too many features (large $|\mathcal{F}_t|$) becomes more severe, thus encouraging more substantial gene reduction. 
Conversely, a lower value of $\lambda$ relaxes the penalty against the size of $|\mathcal{F}_t|$, suitable when minimal reduction is sufficient. 
This metric ensures that the selection process strategically reduces the number of features.

$\star$ \textit{Reward Assignment}: Then we combine two perspectives and
obtain the reward in step-$t$:
\begin{equation} \label{equation:reward_alpha}
    r_t=\alpha \cdot r^p_t + (1-\alpha) \cdot r^q_t,
\end{equation}
where $r_t$ is the total reward in this step. 
$\alpha$ is a hyperparameter to adjust the weight of two perspectives. 
After obtaining the reward, the framework will assign the reward equally to each activated agent. 
This approach of balancing performance and feature number in the reward function is similar to sparse constraint-based feature selection methods (Lasso~\cite{tibshirani1996regression} and LassoNet~\cite{lemhadri2021lassonet}), which incorporate regularization terms into their objective functions, but is more flexible and adaptable.


\subsection{Exploration and Opitmization} \label{sec:hierarchical_reinforced_iteration}
Figure~\ref{main_fig_2} demonstrates our approach of using hierarchical reinforced iteration to discover the best subsets of characteristics, partitioning model training into two distinct phases: exploration and optimization.
\begin{figure}
    \centering
\includegraphics[width=0.55\linewidth]{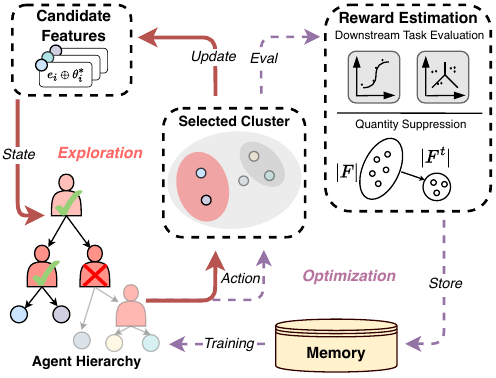}
    \caption{Detail of the iteration and optimization with the hierarchical agents.}
    \label{main_fig_2}
\end{figure}

\smallskip
\noindent\textbf{Exploration Phase.}
In the exploration phase, hierarchical agents randomly explore the feature subspace. 
They take the current state $s_t$ as input and select or drop features hierarchically, constructing a new set of selected features.
The reward $r_t$  and the next state $s_{t+1}$ are computed from the newly selected features subset. 
Each agent $\mathcal{A}_i$ taking action $a_t$ in this step stores the experience $m_t=\{\mathbf{s}_t,a_t,r_t,\mathbf{s}_{t+1}\}$ in its memory $M_i$. 

\smallskip
\noindent\textbf{Optimization Phase.}
Agents will autonomously train their individual policies through the memory mini-batches obtained from prioritized experience replay~\cite{schaul2015prioritized}.
We refined the policy utilizing the Actor-Critic approach~\cite{haarnoja2018soft}, where the policy network $\pi(\cdot)$ assumes the role of the actor, while $V(\cdot)$ serves as its associated critic. The optimization objective for agent $\mathcal{A}_i$ is defined by the expected cumulative reward, expressed as:
\begin{equation}
    \max_{\pi} \mathbb{E}_{m_t \sim \mathcal{B}} \left[ \sum_{t=0}^{T} \gamma_t r_t \right],
\end{equation}

where $\mathcal{B}$ refers to the distribution of experiences within the prioritized replay buffer, $\gamma$ is the discount factor, and $T$ denotes the time horizon of an episode. 
Additionally, we incorporate the Q-function, $Q(\mathbf{s}, a)$, which signifies the expected return for taking action $a$ in state $s$ and adhering to policy $\pi$ subsequently:
\begin{equation}
Q(\mathbf{s}, a) = \mathbb{E} \left[ r + \gamma \max_{a'} Q(\mathbf{s}', a') \mid \mathbf{s}, a \right].
\label{equation:bellman}
\end{equation}
The training adjustments for the actor networks are determined by the policy gradient~\cite{kakade2001natural} defined as
\begin{equation}
 \quad \nabla_{\theta} J(\pi) = \mathbb{E}_{m_t \sim \mathcal{B}} \left[ \nabla_{\theta} \log \pi(a_t | \mathbf{s}_t) A(\mathbf{s}_t, a_t) \right].
\end{equation}
Here, $A(\mathbf{s}, a) = Q(\mathbf{s}, a) - V(\mathbf{s})$ is the advantage function, which aids in the gradient estimation for policy improvement.

\smallskip
\noindent\textbf{Analysis of the Advancement of HRL.} \label{sec:why_hri_fast}
The hierarchical agent structure activates fewer agents in each step, leading to superior performance in the iteration.
The single-agent approach uses one agent to make $N$ decisions on $N$ features. 
Meanwhile, the multi-agent approach uses $N$ agents to make $N$ decisions on $N$ features, and the time complexity of both approaches is $\mathcal{O}(N)$.
In contrast, we apply a hierarchical structure, which reduces the number of activated agents making decisions, totally  using $2N-1$ agents to make $log(N)$ decisions on $N$ features, with time complexity $\mathcal{O}(log(N))$. 
In the following section, we will provide detailed proof of the time complexity advantage. 

\subsection{The Proof of Time Complexity of Comprehend-Divide-and-Conquer Solution}
\label{section:time_complexity_proof}
In this section, we provide formal proofs of the time complexity of the comprehend-divide-and-conquer feature selection solution, analyzing the best, worst, and average cases. Consider an architecture with $N$ agents organized as a complete binary tree, where each non-leaf agent delegates decision-making to its children with probability $p$.

\begin{theorem}[Best-Case Time Complexity]
The best-case time complexity satisfies $\mathcal{O}(1)$.
\end{theorem}

\begin{proof}
The best-case scenario occurs when the root agent terminates the decision process immediately without delegating to its children. This requires only a constant number of operations at the root node:
\begin{equation}
    T_{\text{best}}(N) = c_1
\end{equation}
where $c_1$ is a constant. Therefore,
\begin{equation}
    T_{\text{best}}(N) \in \mathcal{O}(1).
\end{equation}
\end{proof}

\begin{theorem}[Worst-Case Time Complexity]
The worst-case time complexity satisfies $\mathcal{O}(N)$.
\end{theorem}

\begin{proof}
The worst-case scenario occurs when every agent recursively delegates decisions to its children until the leaf nodes are reached, thus activating all $N$ nodes in the tree. For a complete binary tree with $N = 2^h - 1$ nodes (height $h$),
\begin{equation}
    T_{\text{worst}}(N) = \sum_{k=0}^{h-1} 2^k = 2^h - 1 = N.
\end{equation}
Therefore,
\begin{equation}
    T_{\text{worst}}(N) \in \mathcal{O}(N).
\end{equation}
\end{proof}

\begin{theorem}[Average-Case Time Complexity]
Under the assumptions below, the average-case time complexity satisfies $\mathcal{O}(\log N)$ when $p = 1/2$.
\end{theorem}

\begin{assumption}[Balanced Tree Structure]\label{assump:balanced}
The hierarchical architecture is a perfect binary tree with:
\begin{itemize}
    \item $N = 2^h - 1$ nodes, where $h$ is the tree height;
    \item All leaf nodes at the same depth level;
    \item Each non-leaf node has exactly two subtrees of size $\lfloor N/2 \rfloor$.
\end{itemize}
\end{assumption}
While this assumption is theoretically convenient, in practice our hierarchy closely approximates such a structure.
We provide an empirical validation of this assumption in Section~\ref{sec:empirical-hierarchy}, showing that the balance factors and tree heights observed across 21 datasets are strongly aligned with those of a perfect binary tree.

\begin{assumption}[Independent Delegation]\label{assump:delegation}
Each non-leaf agent independently delegates decision-making to its children with probability $p$, where
\begin{equation}
    p \in [0, 1].
\end{equation}
\end{assumption}

Let $E(N)$ denote the expected number of active agents in such a tree. We formalize the analysis as follows:

\begin{lemma}[Expectation Decomposition]
For $N > 1$, the expected number of active agents satisfies
\begin{equation}
    E(N) = (1-p) \cdot 1 + p \cdot \left[1 + E\left(\left\lfloor \frac{N}{2} \right\rfloor\right) + E\left(\left\lceil \frac{N}{2} \right\rceil\right)\right],
\end{equation}
with boundary condition $E(1) = 1$.
\end{lemma}

\begin{proof}
By the law of total expectation:
\begin{align}
    E(N) 
    &= \mathbb{E}[\text{Nodes} \mid \text{root terminates}] \cdot (1-p) \\
    &\quad + \mathbb{E}[\text{Nodes} \mid \text{root delegates}] \cdot p.
\end{align}
Specifically:
\begin{itemize}
    \item If the root terminates: $\mathbb{E}[\text{Nodes}] = 1$;
    \item If the root delegates: $\mathbb{E}[\text{Nodes}] = 1 + E\left(\left\lfloor \frac{N}{2} \right\rfloor\right) + E\left(\left\lceil \frac{N}{2} \right\rceil\right)$.
\end{itemize}
For a perfect binary tree, $\left\lfloor \frac{N}{2} \right\rfloor = \left\lceil \frac{N}{2} \right\rceil = \frac{N}{2}$ when $N = 2^h - 1$, so the recurrence simplifies to
\begin{equation}
    E(N) = 1 + 2p E(N/2).
\end{equation}
\end{proof}

\begin{theorem}[Closed-Form Solution]
For $N = 2^h - 1$, the recurrence resolves to
\begin{equation}
    E(N) = \frac{(2p)^{\log_2(N+1)} - 1}{2p - 1}.
\end{equation}
In particular, when $p = 1/2$,
\begin{equation}
    E(N) = 2\log_2(N+1) + \mathcal{O}(1).
\end{equation}
\end{theorem}

\begin{proof}
Unrolling the recurrence for $N = 2^h - 1$:
\begin{equation}
\begin{split}
E(N) &= 1 + 2pE(N/2) \\
&= 1 + 2p[1 + 2pE(N/4)] \\
&= \sum_{k=0}^{h-1} (2p)^k + (2p)^h E(1) \\
&= \frac{(2p)^h - 1}{2p - 1} + (2p)^h \cdot 1 \\
&= \frac{(2p)^{h+1} - 1}{2p - 1} \\
\end{split}
\end{equation}
Substituting $h = \log_2(N+1)$ yields: 
\begin{equation}
    E(N) = \frac{(2p)^{\log_2(N+1)} - 1}{2p - 1}.
\end{equation}
For $p = 1/2$, $2p = 1$, so: 
\begin{equation}
    E(N) = \log_2(N+1) + 1 = 2\log_2(N+1) + \mathcal{O}(1).
\end{equation}
\end{proof}

Therefore, the average-case time complexity satisfies: 
\begin{equation}
    T_{\text{avg}}(N) \in \mathcal{O}(\log N).
\end{equation}

In Section~\ref{complex_analysis}, we conducted experiments and the results show that the active agents decreased by 70.61\% to 82.30\%, and time consumption reduced by 35.49\% to 55.26\%. These observations are aligned with the conclusion.

\section{Experimental Settings}
\subsection{Dataset Description}
\label{data_des}
We conducted experiments on 21 datasets, which are available for public use from the Feature Selection Benchmark~\cite{feature_selection_benchmark},the National Center for Biotechnology Information (NCBI) ’s
Gene Expression Omnibus (GEO)~\cite{ncbi},UCI~\cite{uci}, Kaggle~\cite{kaggle}, OpenML~\cite{openml}, libSVM~\cite{libsvm}, etc. 
These datasets vary among three tasks (classification, multi-label classification, and regression), sample sizes, and number of feature samples. 
The datasets come from various fields, including biology, finance, and synthetic data.
The datasets contain different sample sizes from 253 to 83773.
These datasets include a wide range of feature number, from 21 to 20670.
To avoid overfitting, we followed existing research~\cite{xiao2023beyond} for splitting the datasets into training and validation folds. 
{
Specifically, reward is evaluated in 80\% of samples; the remaining 20\% of samples are used to validate the efficiency and capabilities of the selected feature.}
Dataset descriptions, including the number of samples, features, and task type, are listed in Table~\ref{main_table}.

\subsection{Evaluation Metrics}
We used F1-Score (Micro-F1), accuracy, and recall to evaluate the performance of classification tasks (multi-classification tasks).
For regression tasks, we utilized the 1-Relative Absolute Error (1-RAE) from studies~\cite{wang2022group,xiao2024traceable} to evaluate the performance.
The higher values indicate better performance for all metrics.

\subsection{Baseline Algorithms}
\label{base_alg}
We conduct extensive comparisons against 8 methods, encompassing both classical and recent approaches, covering three categories of feature selection: filter, embedded, and RL-based (latest warpper methods).
For filter methods, KBest~\cite{yang1997comparative} selects the K-best features based on statistical measures, and MCDM~\cite{mcdm} uses predefined rankers to assign scores to features and selects them accordingly.
for embedded methods, LASSONet~\cite{lemhadri2021lassonet} utilizes a residual component that integrates feature selection with the model training process, enhancing its generalizability, and GAINS~\cite{xiao2023beyond} embeds features into a vector space and identifies the optimal subset using a gradient ascent search algorithm; for wrapper methods (including reinforcement learning-based ones), CompFS~\cite{yasuda2023sequential} provides an efficient, single-pass implementation of greedy forward selection, utilizing attention weights at each step to approximate feature importance, SAFS~\cite{imrie2022composite} ensembles feature selection models within a neural network to identify feature groups and measures the similarity between these groups to select the optimal feature subset.
For RL-based methods, MARLFS~\cite{marlfs} is a multi-agent system in which the number of agents matches the number of features, and RLAS~\cite{RLAS} synergizes median-initialized filters, Q-learning-based module selection, and importance-aware random grouping.

\subsection{Hyperparameter Settings and Reproducibility}
We performed 200 epochs for \model\ to explore the feature space and an additional 200 epochs to optimize the agents, with the memory size set to 400.
Follow the hyperparameter and ablation experiment results, The policy model of all agents are Actor-Critic, where the actor and critic models are both implemented as two-layer neural networks, with (64, 8) as hidden size. 
We employed Adam to optimize actor and critic models with a minibatch size of 32.
Following the existing research~\cite{wu2020finite}, to ensure convergence, the critic's learning rate should be around ten times greater than that for the actor. So the learning rates of the actor and critic models were 0.001 and 0.01, respectively.
The hyperparameter $\gamma$ in Equation~\ref{equation:bellman} as 0.9.
We adopted Random Forest as the downstream machine learning model to evaluate the performance of the selected feature subset in each step.
We set all features' Gaussian component number $k$ as the maximum number of their BIC search result.
For feature metadata embedding, we apply \texttt{text-embedding-3-large}~\cite{openai_embedding} based on GPT-4~\cite{openai_gpt4}.
We set the hyperparameters $\alpha$ in Equation~\ref{equation:reward_alpha} and $\lambda$ in Equation~\ref{equation:reward_lambda} as 0.4 and 0.6 based on the results from the hyperparameter study in section~\ref{hyper_exp}.
The hyperparameter settings for the baselines follow their original articles. 

\subsection{Environmental Settings} 
The experiments were conducted on an Ubuntu 18.04.6 LTS operating system, equipped with an AMD EPYC 7742 CPU and 4 NVIDIA V100 GPUs, within the Python 3.11.0 environment and using PyTorch 2.1.1.

\section{Experimental Results}

\begin{table*}[h]
\vspace{-0.2cm}
\caption{Overall performance comparison. The best result is highlighted in \textbf{bold} for each dataset, and the second-best result is highlighted \uline{underlined}.  \textsc{\#Samp} and \textsc{\#Feat} denote the number of samples and features.}
\vspace{-0.2cm}
\label{main_table}
\resizebox{\linewidth}{!}{\setlength\tabcolsep{2.5mm}
\begin{tabular}{lcccccccccccc}
\toprule 
Dataset         & Task & \#Samp. & \#Feat. & $\text{KBest}$ & $\text{LASSONet}$ & $\text{MCDM}$  & $\text{RLAS}$ & $\text{MARLFS}$ & $\text{SAFS}$  & $\text{GAINS}$ & $\text{CompFS}$ & Ours  \\ 
\midrule
SpectF          & C    & 267     & 44      & 79.16          & 82.11            & 79.44          & 80.38         & 77.01           & 80.78          & 80.55          & \underline{83.69}  & $\mathbf{87.56^{\pm 0.22}}$ \\
SVMGuide3       & C    & 1243    & 21      & 76.55          & \underline{77.61} & 75.68          & 76.13         & 75.92           & 75.53          & 77.53          & 72.83            & $\mathbf{78.90^{\pm 0.45}}$  \\
German\_Credit  & C    & 1001    & 24      & 65.75          & 58.85            & \underline{70.49} & 64.91       & 66.89           & 66.89          & 69.40          & 68.27            & $\mathbf{73.07^{\pm 1.13}}$  \\
Credit\_Default & C    & 30000   & 25      & 80.40 & 79.86          & 74.87          & \underline{80.41}         & 80.11           & 77.49          & 80.05          & 78.15            & $\mathbf{80.56^{\pm 0.09}}$  \\
SpamBase        & C    & 4601    & 57      & 90.36          & 89.16            & 89.10          & 91.40         & 90.04           & 90.93          & \underline{91.71} & 89.70          & $\mathbf{92.60^{\pm 0.31}}$  \\
Megawatt1       & C    & 253     & 38      & 84.02          & 89.59            & 89.59          & 83.22         & 87.64           & 86.46          & 89.42          & \underline{91.83}  & $\mathbf{92.70^{\pm 0.73}}$  \\
Ionosphere      & C    & 351     & 34      & 91.48          & 90.17            & 88.85          & \underline{94.11}         & 90.18           & 90.18          & 92.24          & 91.08            & $\mathbf{94.27^{\pm 1.04}}$  \\ 
\midrule
Mice-Protein    & MC   & 1080    & 77      & \underline{82.86} & 82.71          & 81.95          & 77.78         & 80.56           & 77.81          & 81.04          & 78.96            & $\mathbf{83.79^{\pm 0.91}}$  \\
Coil-20         & MC   & 1440    & 400     & 93.92          & 88.19            & 96.53          & 90.97         & 96.19           & 94.41          & \underline{97.22} & 95.16          & $\mathbf{97.56^{\pm 0.06}}$  \\
MNIST           & MC   & 10000   & 784     & 89.67          & 86.40            & 91.25          & 88.30         & 90.75           & 86.65          & \underline{91.40} & 87.65          & $\mathbf{91.75^{\pm 0.24}}$  \\
Otto            & MC   & 61878   & 93      & 72.31          & 71.65            & 73.24          & 72.86         & \underline{73.74} & 72.13          & 72.40          & 71.97            & $\mathbf{74.24^{\pm 0.07}}$  \\ 
Jannis          & MC   & 83733   & 54      & 65.48          & 65.78            & 64.20          & 57.88         & \underline{66.70} & 64.42          & 65.15          & 66.61            & $\mathbf{67.91^{\pm 0.22}}$  \\ 
Cao             & MC   & 4186    & 13488   & 82.81          & 83.65            & 82.37          & 84.57         & 83.71           & 84.67          & 82.47          & \underline{85.57}  & $\mathbf{89.37^{\pm 0.13}}$  \\ 
Han             & MC   & 2746    & 20670   & 70.18          & 72.57            & 71.33          & 74.73         & 74.69           & \underline{75.79} & 74.03          & 73.20            & $\mathbf{81.45^{\pm 0.21}}$  \\ 
\midrule
Openml\_586     & R    & 1000    & 25      & 56.35          & 56.03            & 55.82          & 51.04         & 55.02           & 53.15          & \underline{59.36} & 55.49          & $\mathbf{61.75^{\pm 0.71}}$  \\
Openml\_589     & R    & 1000    & 25      & 53.34          & 52.99            & 54.12          & 49.36         & 52.68           & 45.51          & \underline{59.33} & 48.46          & $\mathbf{61.60^{\pm 1.01}}$  \\
Openml\_607     & R    & 1000    & 50      & 53.63          & 53.70            & 55.37          & 57.42         & 54.82           & 51.82          & \underline{60.44} & 53.18          & $\mathbf{61.70^{\pm 0.69}}$  \\
Openml\_616     & R    & 500     & 50      & 31.16          & 21.69            & 28.27          & 50.38         & 31.36           & 46.16          & \underline{49.54} & 49.16          & $\mathbf{53.90^{\pm 0.86}}$  \\
Openml\_618     & R    & 1000    & 50      & 49.21          & 48.97            & 48.69          & 58.57         & 49.31           & 45.16          & \underline{55.68} & 48.55          & $\mathbf{56.40^{\pm 0.62}}$  \\
Openml\_620     & R    & 1000    & 25      & 53.00          & 54.63            & 55.00          & 41.63         & 54.57           & 50.16          & \underline{61.51} & 55.49          & $\mathbf{62.78^{\pm 0.82}}$  \\
Openml\_637     & R    & 500     & 50      & 23.09          & 25.42            & 23.75          & 42.04         & 25.96           & 34.18          & \underline{39.46} & 36.17          & $\mathbf{40.54^{\pm 0.37}}$  \\
\bottomrule
\vspace{-0.6cm}
\end{tabular}}
 \begin{tablenotes}
  \item {\small * We report F1-Score for classification and multi-class classification, 1-RAE for regression.} 
    \item {{\small ** The standard deviation is computed based on the results of 5 independent runs}}
\end{tablenotes}
\vspace{-0.6cm}
\end{table*}

\subsection{Main Comparison}
\label{main_exp}
This experiment aims to answer the question: \textit{
Is \model\ capable of effectively refining a superior feature subspace across all domains?}
We compared our method with classical and recent approaches across various datasets. Table~\ref{main_table} shows the overall results of \model{}.
We observed that \model\ outperforms the classical, deep learning-based, and RL-based methods on all tasks.
The underlying driver is that, compared to the classical methods (KBest, LASSONet, and MCDM), \model\ discovers fine-grained differentiation between features through LLMs to construct a more accurate feature subspace beyond numerical statistics, resulting in better performance.
Compared to deep learning-based methods (SAFS, CompFS, GAINS), \model\ manages clustered features using a hierarchical architecture that considers the correlation between features from coarse to fine grain, enables better exploration of feature subspace.
Compared to RL-based methods (RLAS, MARLFS), agent collaboration is viewed through the hierarchical reinforcement learning architecture in \model{}, leading to an improvement in performance.
Overall, this experiment illustrates that \model\ is prominent and robust across diverse datasets, emphasizing various tasks, sizes, and fields, underlining its universal applicability for feature subspace exploration.

\subsection{Space/Time Complexity Analysis}

\noindent\textbf{Space/Time Complexity Comparison with OAPF.}
This experiment aims to answer the question:
\textit{Does \model\ have an advantage in space/time complexity over One-Agent-Per-Feature (OAPF) approaches?}
Table~\ref{exp:space} compares the number of active agents as well as the overall runtime between the OAPF architecture and \model{} on four datasets.
We observed that the number of active agents for \model\ is reduced by at least 70.61\% across all four datasets. Furthermore, this reduction in active agents directly translates to substantial time savings: on every reported dataset, the overall runtime of \model{} decreases by more than 30\% to 50\% compared to OAPF.
These empirical results corroborate our complexity proof in Section~\ref{section:time_complexity_proof}, which shows that the expected number of active agents for \model\ is $\mathcal{O}(\log N)$, in contrast to the $\mathcal{O}(N)$ active agents required by OAPF. The underlying reason is that \model's hierarchical decision structure, built upon the divide-and-conquer paradigm, achieves superior decision efficiency, significantly reducing the number of active agents and thus both space and time complexity.
In summary, by decreasing the number of active agents, the hierarchical architecture not only achieves better space efficiency but also leads to a dramatic reduction in runtime compared to OAPF approaches.

\label{complex_analysis}
\begin{table*}[htbp]
\centering
\caption{
 Comparison of the average number of active agents between OAPF (one agent per feature) and our model. 
}
\label{exp:space}
\resizebox{0.95\linewidth}{!}{\setlength\tabcolsep{2.5mm} \small
\begin{tabular}{lcccccccc}
\toprule
Dataset & \multicolumn{2}{c}{Spam Base} & \multicolumn{2}{c}{Mice-Protein} & \multicolumn{2}{c}{Coli-20} & \multicolumn{2}{c}{MNIST} \\ \midrule
Method              & Ours$^{-h}$ & Ours  & Ours$^{-h}$ & Ours  & Ours$^{-h}$ & Ours   & Ours$^{-h}$ & Ours    \\ \midrule
Active Agents (Avg.) & 56     & $16.46^{-70.61\%}$ & 76     & $13.45^{-82.30\%}$ & 400    & $104.39^{-73.90\%}$ & 784    & $156.73^{-80.01\%}$  \\ \midrule
\textcolor{black}{Time Consumption (s)}      & 146.56   & $76.28^{-47.95\%}$ & 177.77   & $79.53^{-55.26\%}$ & 547.05  & $352.86^{-35.49\%}$ & 2291.25  & $1317.28^{-42.50\%}$ \\ \bottomrule
\end{tabular}
}
\vspace{-0.5cm}
\end{table*}

\begin{figure*}[htbp]
\centering 
\begin{subfigure}{0.245\textwidth}
\includegraphics[width=\textwidth]{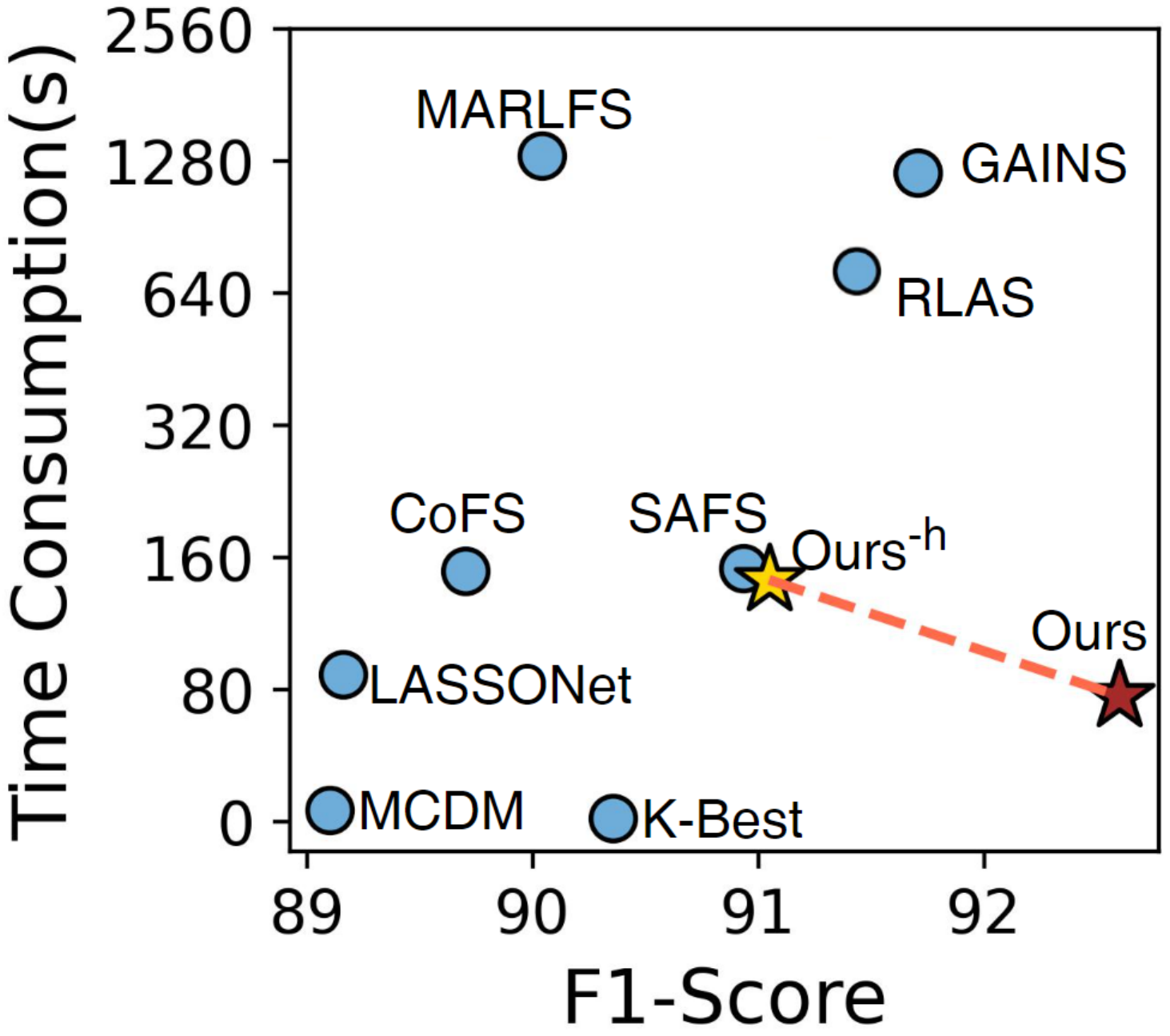}
\caption{Spam Base}
\end{subfigure}
\begin{subfigure}{0.24\textwidth}
\includegraphics[width=\textwidth]{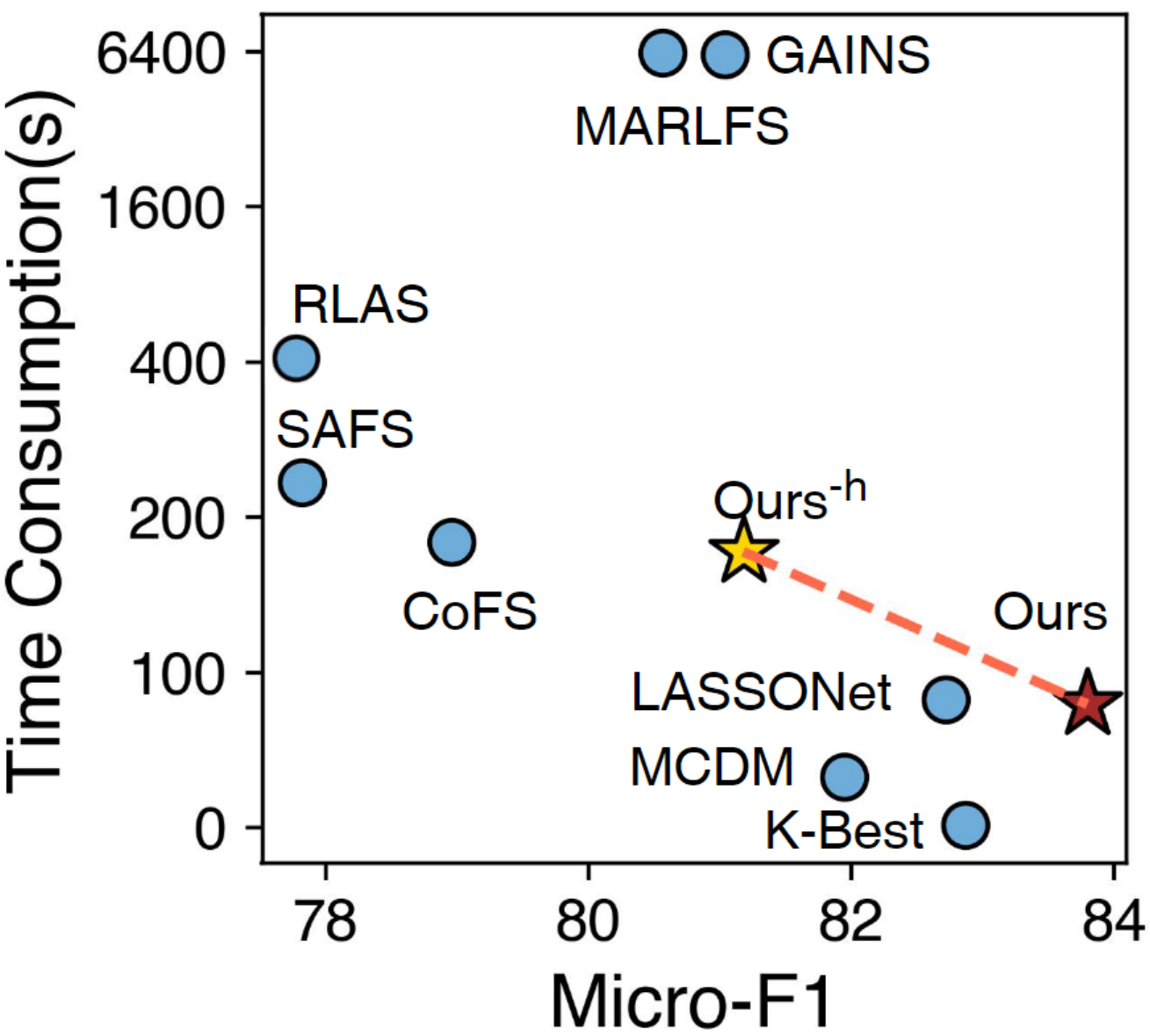}
\caption{Mice-Protein}
\end{subfigure}
\begin{subfigure}{0.24\textwidth}
\includegraphics[width=\textwidth]{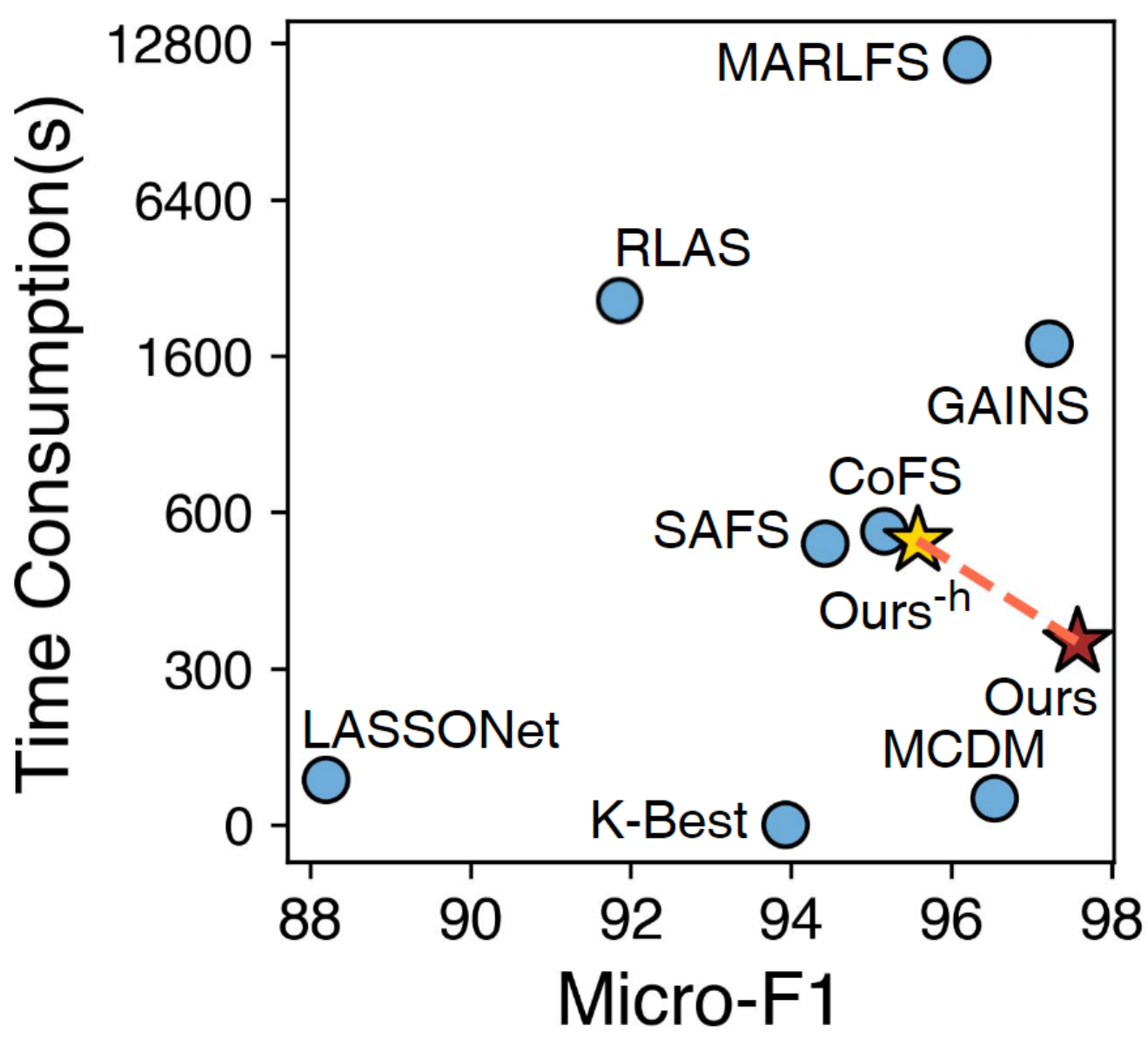}
\caption{Coil-20}
\end{subfigure}
\begin{subfigure}{0.24\textwidth}
\includegraphics[width=\textwidth]{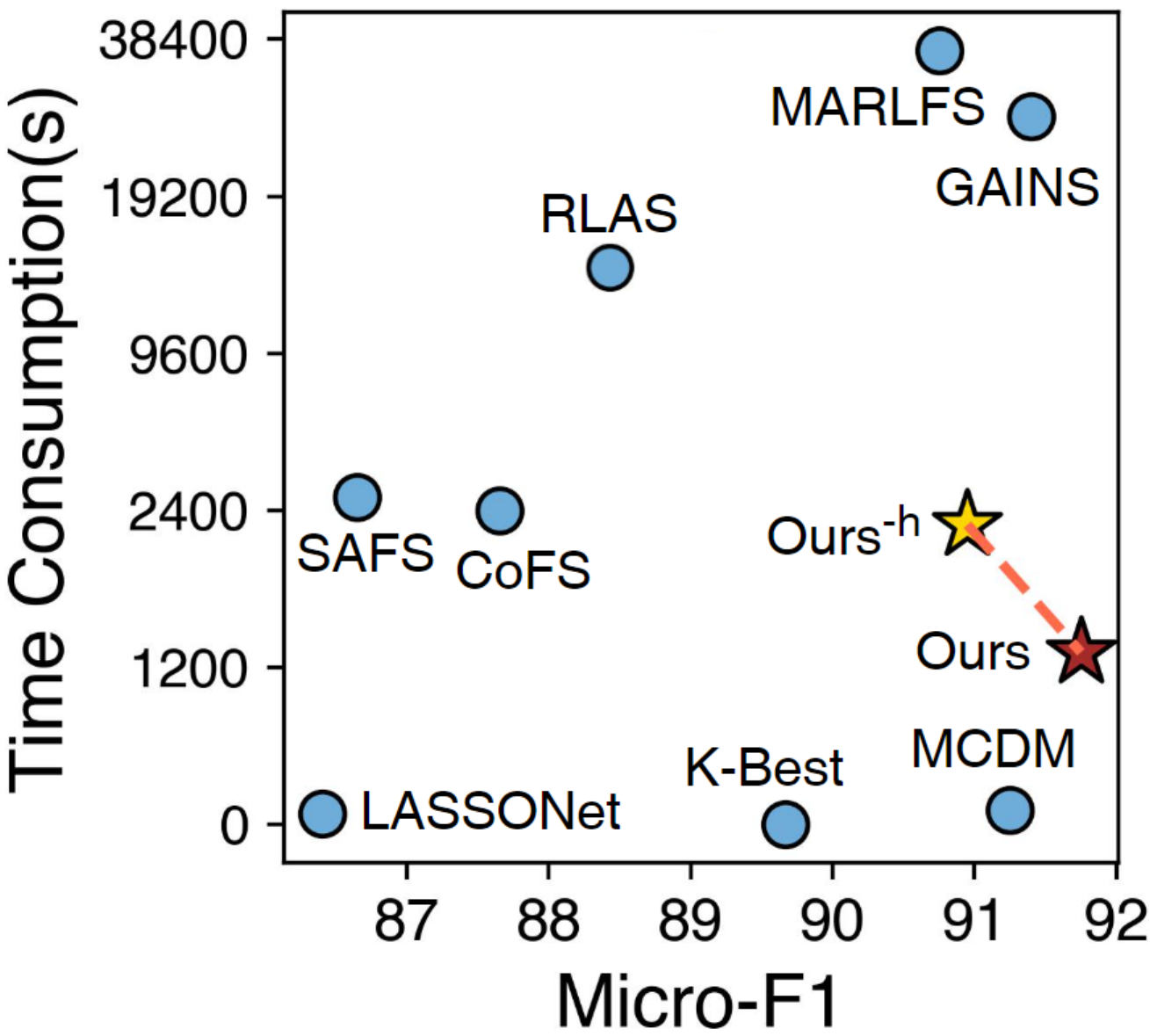}
\caption{MNIST}
\end{subfigure}
\caption{Comparison of \model\ and baseline methods in downstream task performance and time consumption.}
\vspace{-0.3cm}
\label{exp:runtime}
\end{figure*}

\noindent\textbf{Space/Time Complexity Comparison with Baselines.}\label{sec:space_time_complexity_comparison}
This experiment aims to answer the question: \textit{Does \model\ excel in both time and performance?}
Figure~\ref{exp:runtime} shows the performance and time consumption of \model\ compared to other baseline methods and our ablation model $\model^{-h}$ which replaces the multi-agent hierarchical reinforcement learning architecture with One-Agent-Per-Feature (OAPF) architecture.
We found that \model\ outperforms all baselines and $\model^{-h}$, with time consumption comparable to traditional statistical methods (KBest, MCDM, and LASSONet) and much superior to methods based on RL (RLAS and MARLFS), deep learning-based methods (SAFS, CoFS, and GAINS), and $\model^{-h}$.
Compared to the deep learning-based methods, our agents only consists of two-layer neural networks, which are less time-consuming compared to numerous-parameters models used by these methods.
Compared to RL-based methods, \model\ applies the hierarchical decision architecture to manage clustered features, greatly reducing the number of active agents, where active agents will make decisions when the rest of the agents remain inactive, enabling more efficient iteration and less time consumption.
Further, \model\ far outperforms $\model^{-h}$ in both performance and time, underlining the fact that the hierarchical decision architecture can capture correlations between features, facilitate agents’ collaboration, and reduce the number of active agents.
In conclusion, with effective state extraction and hierarchical architecture design, \model\ offers great advantages in terms of time and performance.

\subsection{Ablation Study}\label{abl_exp}
We designed three variants: \textbf{$\model^{-s(GMM)}$} which only utilizes the Gaussian Mixture Model for feature representation; and \textbf{$\model^{-s(LLM)}$}, which only employs the Large Language Model for feature representation;
\textbf{$\model^{-h}$}, which replaces the multi-agent hierarchical reinforcement learning architecture with OAPF architecture.


\begin{figure*}[!h]
    \centering
    \begin{subfigure}[b]{0.245\textwidth}
        \centering
        \includegraphics[width=\linewidth]{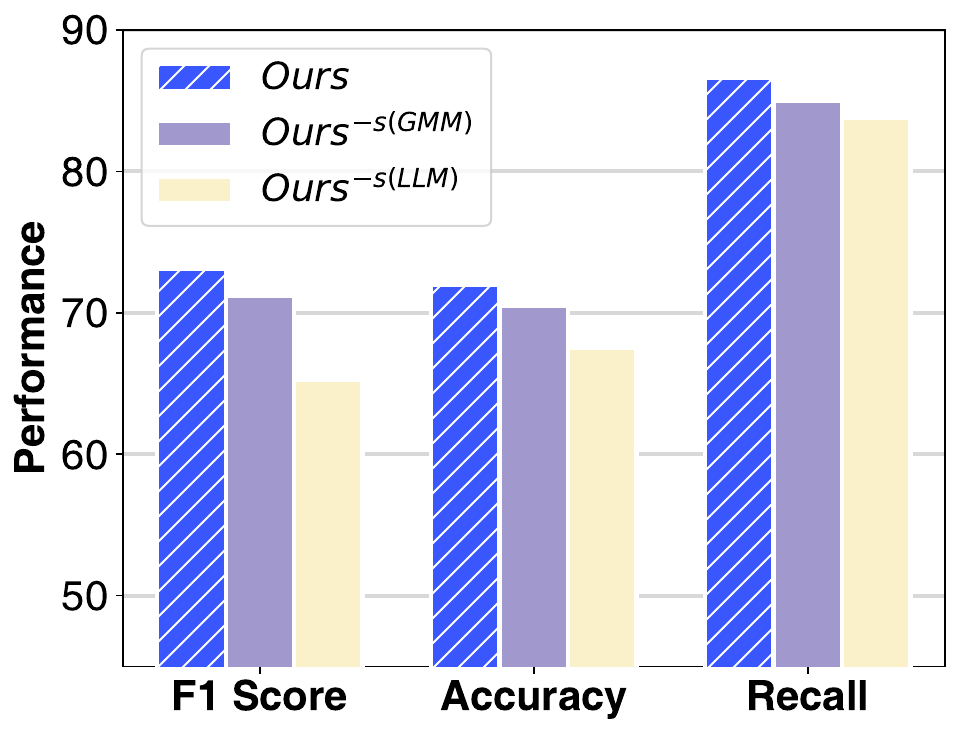}
        \caption{German Credit}
    \end{subfigure}
    \begin{subfigure}[b]{0.24\textwidth}
        \centering
        \includegraphics[width=\linewidth]{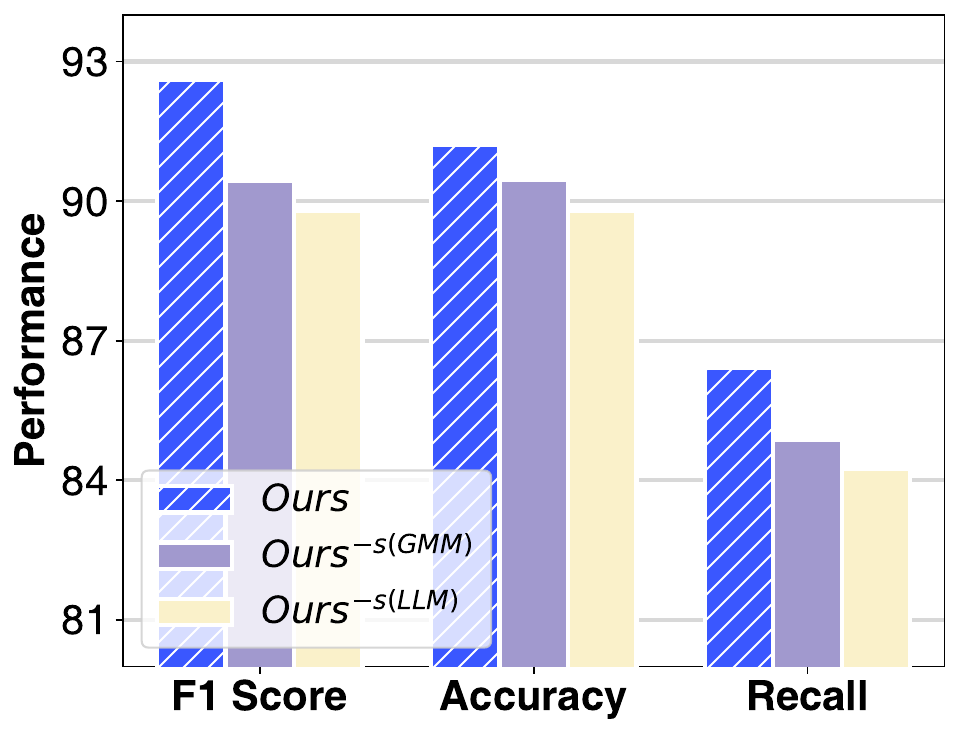}
        \caption{Spam Base}
    \end{subfigure}
    \begin{subfigure}[b]{0.24\textwidth}
        \centering
        \includegraphics[width=\linewidth]{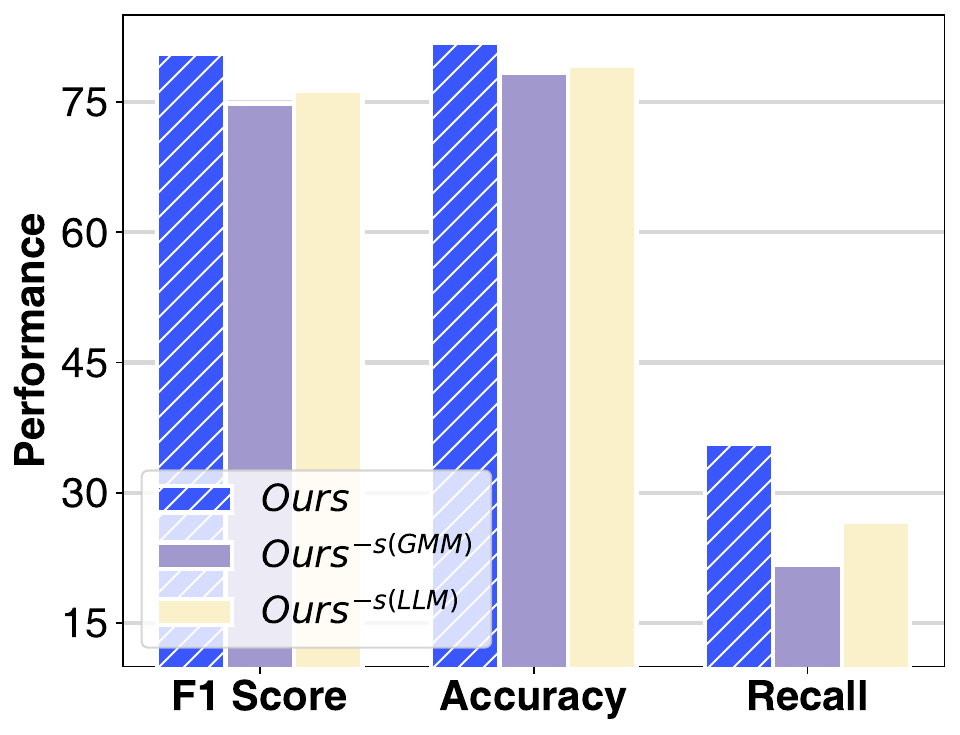}
        \caption{Credit Default}
    \end{subfigure}
    \begin{subfigure}[b]{0.24\textwidth}
        \centering
        \includegraphics[width=\linewidth]{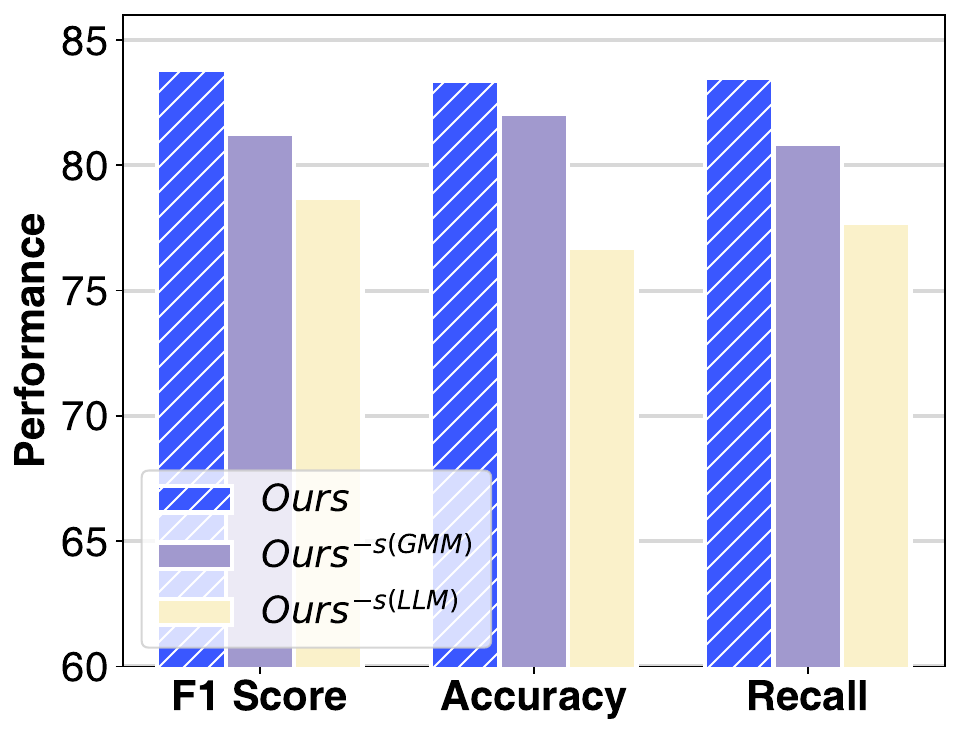}
        \caption{Mice-Protein}
    \end{subfigure}
    \caption{The impact of different state representation methods.}
    \vspace{-0.4cm}
    \label{fig:exp_abl_s}
\end{figure*}

\noindent\textbf{Impact of Hybrid State Representation.}
\label{experiment:hybrid_state_representation}
Figure~\ref{fig:exp_abl_s} shows the comparison results of feature representation variants on two classification tasks and two regression tasks.
We find that \model\ is overall better than $\model^{-s(GMM)}$ and $\model^{-s(LLM)}$.
The underlying driver for this observation is that, compared to $\model^{-s(GMM)}$, the hybrid state representation utilizes the power of LLMs to interpret and extract information from feature metadata, providing valuable insights for hierarchical agents to optimize and explore feature subspace, leading to a better feature selection.
Compared to $\model^{-s(LLM)}$, the hybrid state representation component integrates GMM statistical data, providing a better description of the distribution of each feature, resulting in a performance improvement.
In conclusion, this experiment illustrates that the hybrid state representation component incorporates meaningful information from feature names and statistical data, playing an important role in \model{}.


\begin{figure}[!h]
\centering 
\vspace{-0.2cm}
\begin{subfigure}{0.48\textwidth}
\includegraphics[width=\textwidth]{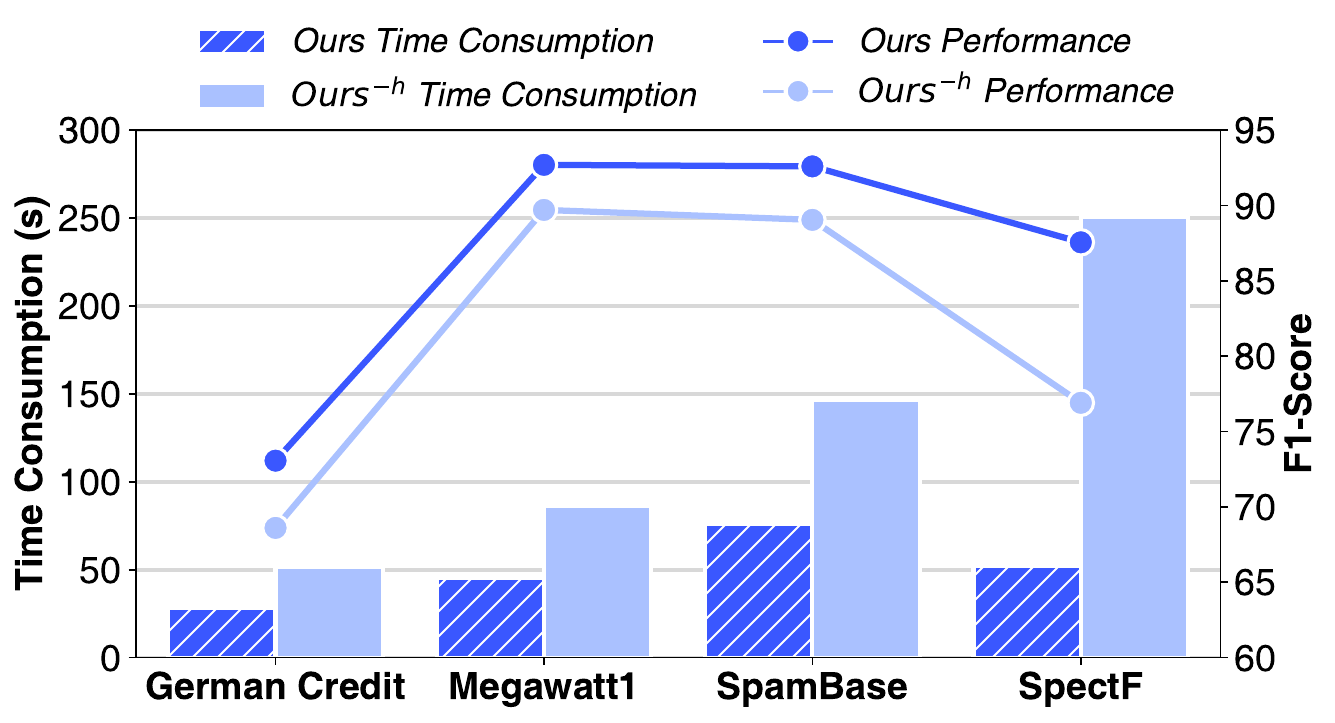}
\end{subfigure}
\begin{subfigure}{0.48\textwidth}
\includegraphics[width=\textwidth]{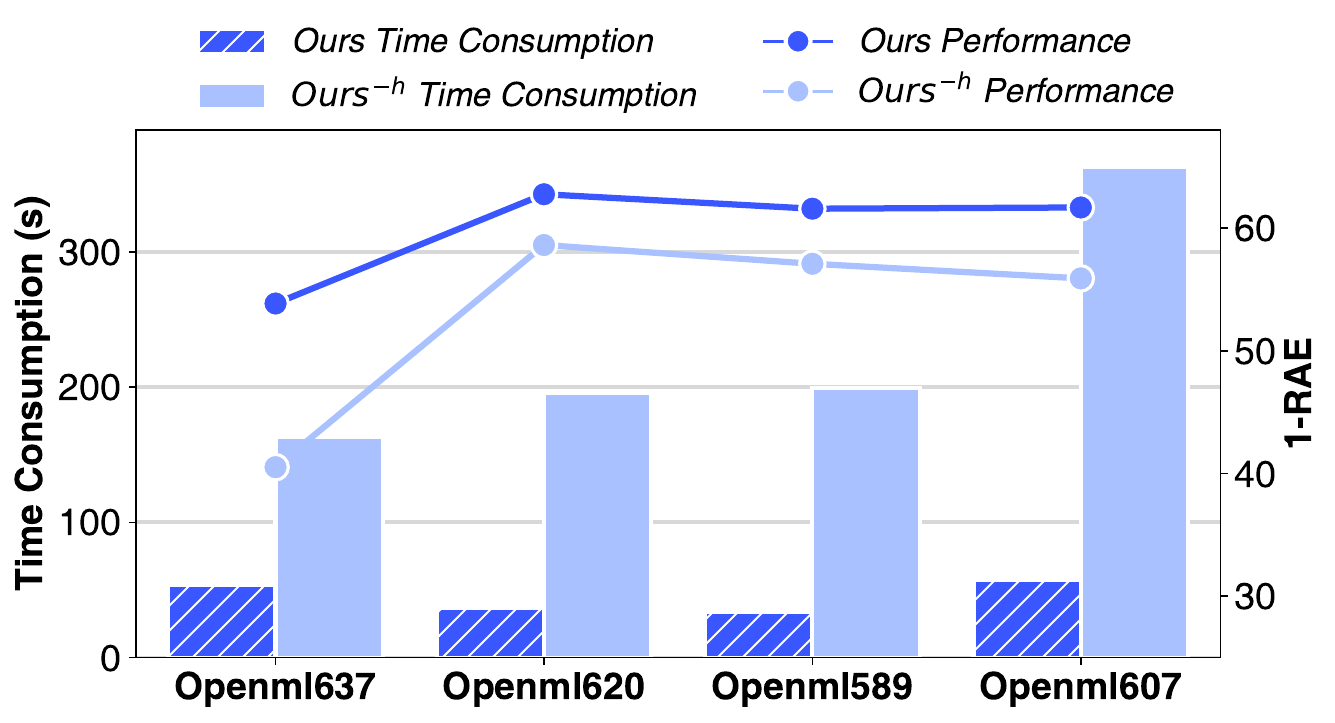}
\end{subfigure}
\caption{The impact of multi-agent hierarchical reinforcement learning architecture.}
\vspace{-0.4cm}
\label{fig:abl_h}
\end{figure}

\smallskip
\noindent\textbf{Impact of HRL Architecture.}
Figure~\ref{fig:abl_h} compares HRL architecture variants on two classification tasks and two regression tasks.
We can observe that \model\ reduces time consumption while outperforming $\model^{-h}$ across performance. 
This implies that the hierarchical agent structure greatly reduces the number of active agents, thus reducing the consumed time. 
Meanwhile, this hierarchical architecture provides a variable granularity toward feature subspace exploration, bringing about improved performance.

\begin{figure}[!ht]
\centering 
\begin{subfigure}{0.33\textwidth}
\includegraphics[width=\textwidth]{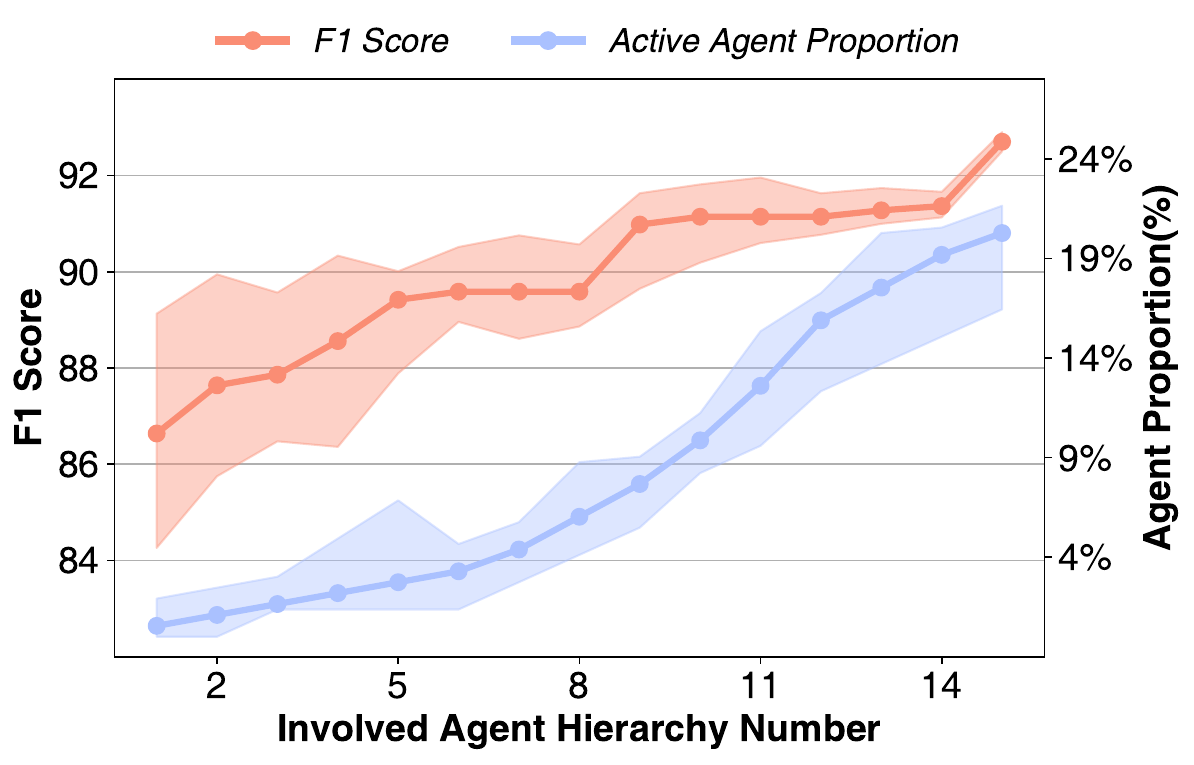}
\caption{Megawatt1}
\end{subfigure}
\begin{subfigure}{0.33\textwidth}
\includegraphics[width=\textwidth]{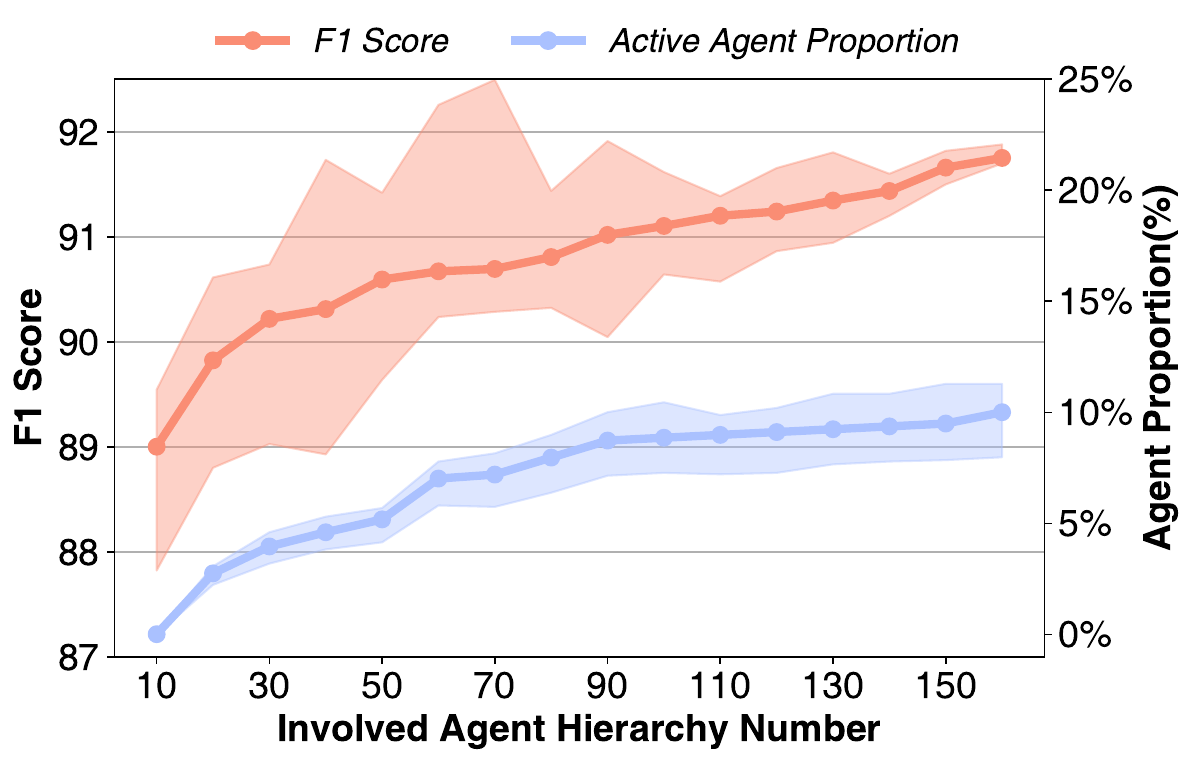}
\caption{MNIST}
\end{subfigure}
\begin{subfigure}{0.33\textwidth}
\includegraphics[width=\textwidth]{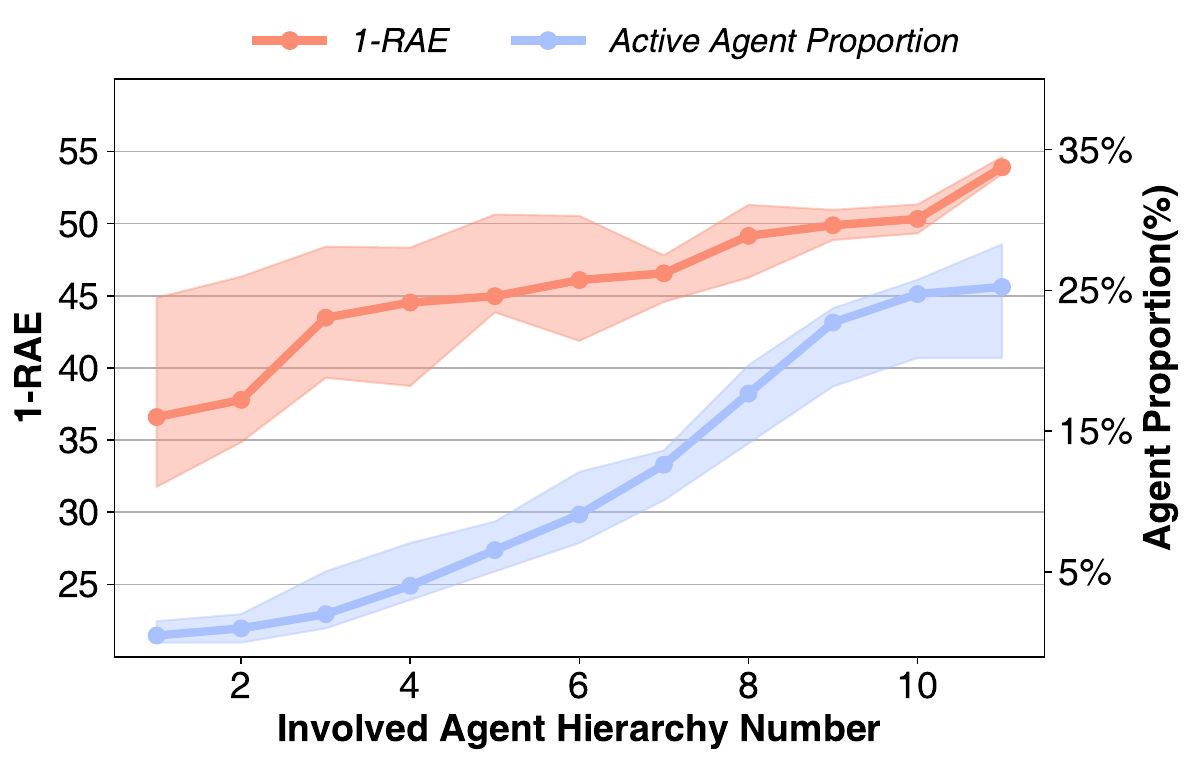}
\caption{OpenML\_616}
\end{subfigure}
\caption{The performance on downstream tasks and the proportion of active agents with changes in the granularity of the agent hierarchy involved.}
\label{fig:gran}
\end{figure}

\subsection{Analysis of the Multi-Agent Hierarchy Architecture Setting}\label{Granularity}
This experiment aims to answer the question: \textit{How does a finer or coarser-grained Agent Hierarchy setting impact the overall system?}
In Algorithm~\ref{algorithm:feature_aggregation}, features with similar semantic or mathematical distributions are controlled by the same agent, forming a hierarchical tree structure. 
In this experiment, we selected one dataset from each of the three tasks—Megawatt1, MNIST, and OpenML\_616—and varied the involved agent hierarchy number. 
A smaller value indicates that fewer agent layers participate in feature selection, leading to a coarser-grained decision-making process.

From the result in Figure~\ref{fig:gran}, we can observe that increasing the number of hierarchy levels leads to consistent improvement of performance and stabilization, as evidenced by higher scores and reduced standard deviation. 
The underlying driver of this behavior is the finer exploration of the feature space, facilitated by the increased number of agents in more granular hierarchies. 
Further, we note that coarser hierarchies, which activate fewer agents, result in lower performance due to a less detailed exploration of relevant features. 
This suggests that a more granular hierarchy enables better agent coordination and decision-making, improving feature selection efficiency. 
However, this improvement comes at the cost of higher computational complexity and more active agents, indicating a trade-off between agent involvement and system efficiency.

\subsection{Analysis of the Hybrid State Extraction Method}
\label{experiment:state_extraction}
This experiment aims to answer the question: \textit{Is combining GMM and GPT-4 for feature state extraction optimal?}
\begin{figure}[!h]
\centering
\includegraphics[width=0.65\linewidth]{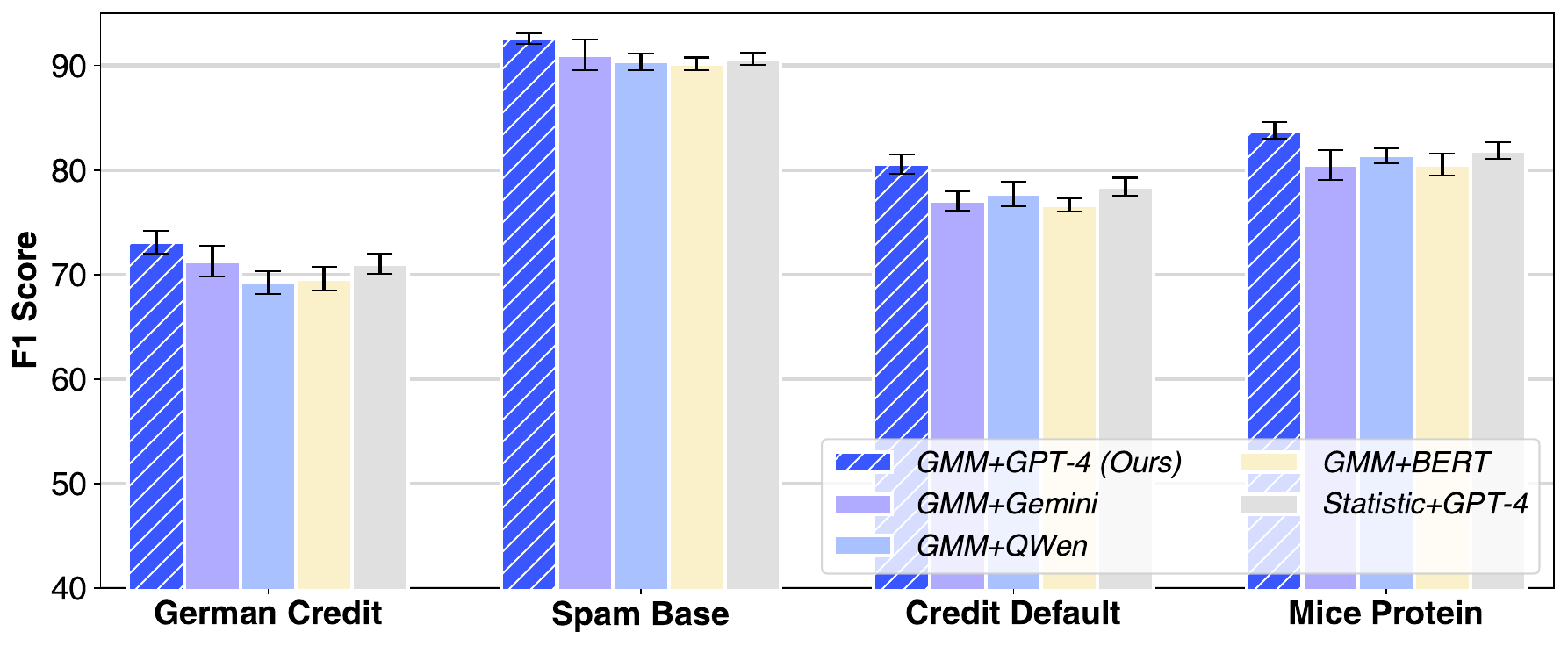}
\caption{Comparison of different feature state extraction models.}
\label{fig:feature_state_extraction}\vspace{-0.4cm}
\end{figure}
We compared various feature distribution extraction methods and semantic extraction models with different parameter scales in several datasets, shown in Figure~\ref{fig:feature_state_extraction}. \textbf{Statistic+GPT-4} applies the mean, maximum, minimum, variance, the first quartile, the second quartile, and the third quartile to extract feature distribution and GPT-4 to leverage metadata. \textbf{GMM+BERT, GMM+QWen, GMM+Gemini} use GMM to pick up statistic and leverage semantic information using BERT~\cite{bert}, \texttt{gte-Qwen2-7B-instruct}~\cite{gte-Qwen2-7B-instruct}, \texttt{gemini-embedding-exp} ~\cite{gemini-embedding-exp-03-07}, respectively. 
By comparing \textbf{Statistic+GPT-4} and \textbf{GMM+GPT-4}, we can observe that the feature distribution obtained using GMM has better results compared to statistics. 
The underlying driver of this behavior is that statistics such as mean, maximum, minimum, etc., can only provide a simple description of the data and cannot capture the inner structure of the data, whereas GMM can approximate an arbitrarily shaped probability distribution by a linear combination of multiple Gaussian distributions, which better captures this complexity and leads to better results.
By comparing \textbf{GMM+BERT, GMM+QWen, GMM+Gemini, GMM+GPT-4}, we find that GPT-4 extracts semantic information with better performance than other models. 
This implies that compared to BERT and QWen, the model architecture and scale of GPT-4 are larger and capable of capturing more complex semantic relationships and contextual information. Compared to Gemini, GPT-4's data screening and cleaning criteria reduce errors and biases in training data, further improving the accuracy and reliability of semantic information extraction.
In summary, as GMM and GPT-4 can accurately extract feature distribution and semantic information from data, they are the best choice for feature state extraction.

\subsection{Analysis of the Feature Clustering Algorithm Selection}
\label{exp:clustering_model}
This experiment aims to answer the question: \textit{Whether the H-clustering algorithm has superior performance compared to other clustering methods in the feature subspace exploration task?} We compared the H-clustering (Ours) with KMeans~\cite{Kmeans},  DBSCAN~\cite{DBSCAN} and SpectralClustering~\cite{spectral_clustering} in four datasets with different quantitative numbers of features in Figure~\ref{fig:exp_cluster_model}.

\begin{figure}[!h]
    \centering
        \centering
\includegraphics[width=0.65\linewidth]{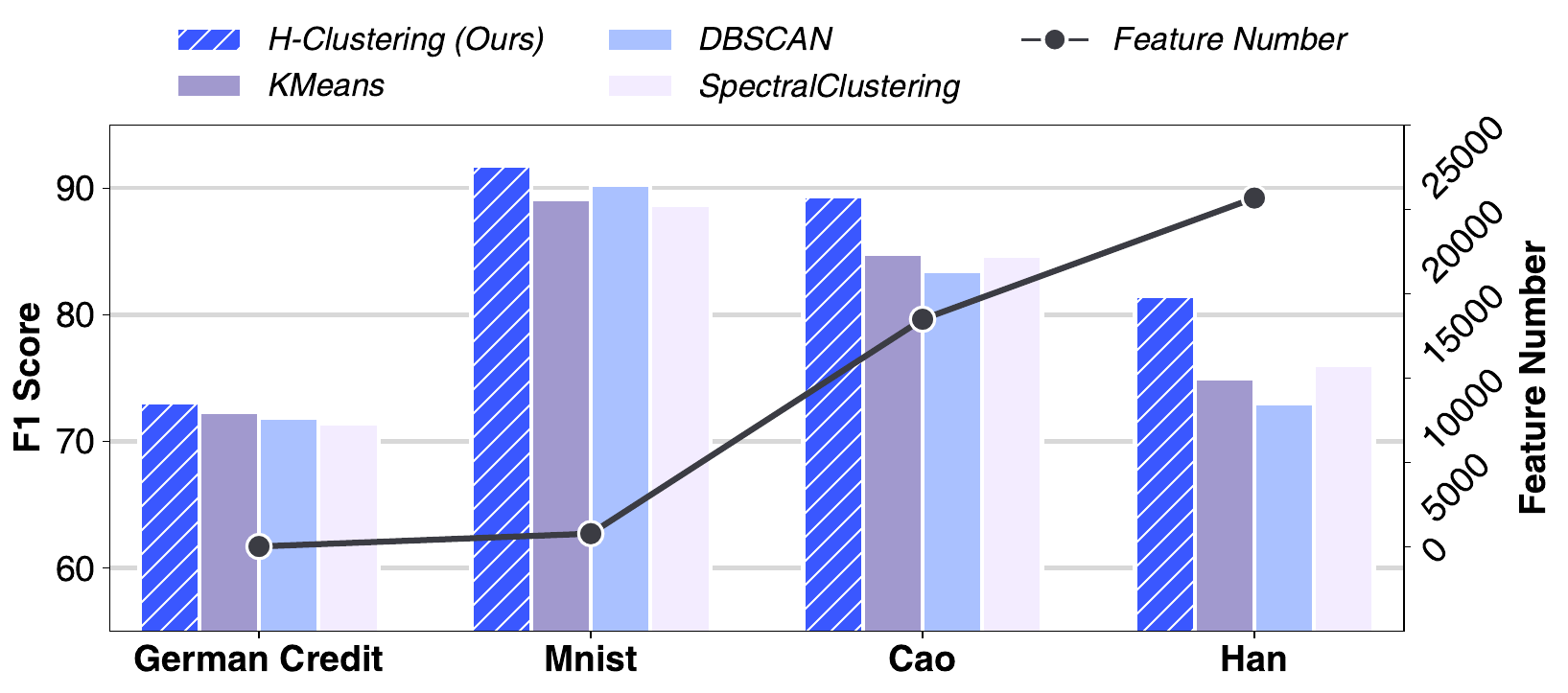}
\caption{{The impact of different clustering algorithms.}}
\label{fig:exp_cluster_model}
\vspace{-0.4cm}
\end{figure}

We can observe that the H-clustering outperforms all clustering algorithms, and the advantages of H-clustering become more apparent as the number of features increases.
The underlying driver of this behavior is that, compared to ordinary clustering algorithms, H-clustering aggregates features from fine-grained to coarse-grained, constructs a hierarchical decision architecture, and considers the feature selection problem from more granularities, leading to better performance. At the same time, this coarse-to-fine decision-making approach decomposes an elaborate task into several relatively simple tasks, which naturally has the advantage of dealing with complex problems, so H-clustering is more advantageous for tasks with a larger number of features.
In summary, H-clustering decides feature selection problems at different granularities, has better performance, can handle more complex problems, thus outperforming other methods.

\begin{figure*}[!h]
    \centering
    \begin{subfigure}[b]{0.33\textwidth}
        \centering
\includegraphics[width=\linewidth]{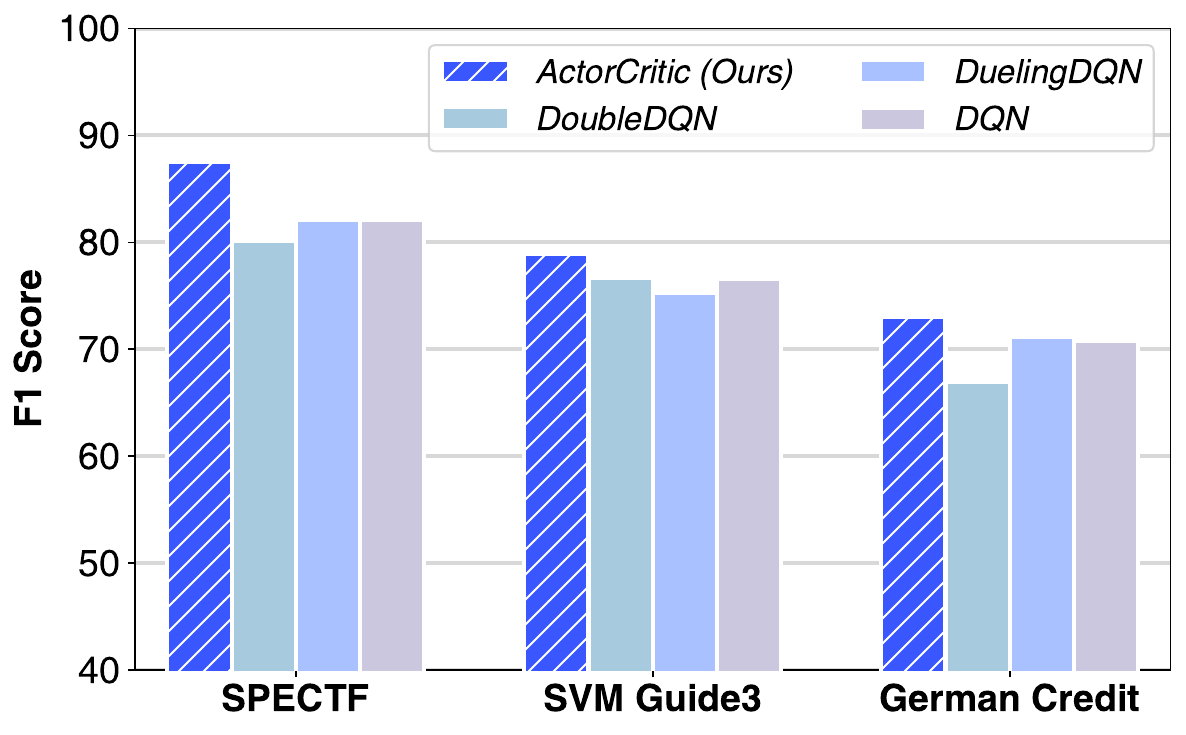}
        \caption{Classification Task}

    \end{subfigure}
    \begin{subfigure}[b]{0.33\textwidth}
        \centering
\includegraphics[width=\linewidth]{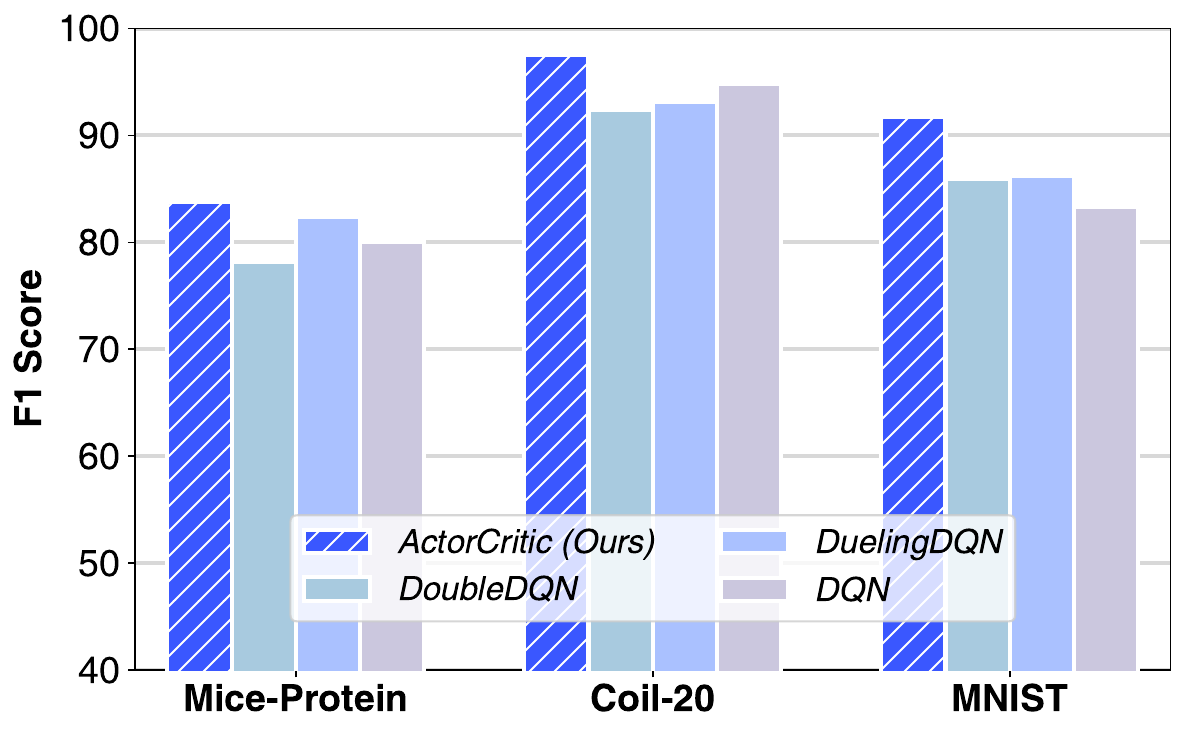}
\caption{Multi-class Classification Task}
\end{subfigure}
    \begin{subfigure}[b]{0.33\textwidth}
        \centering
\includegraphics[width=\linewidth]{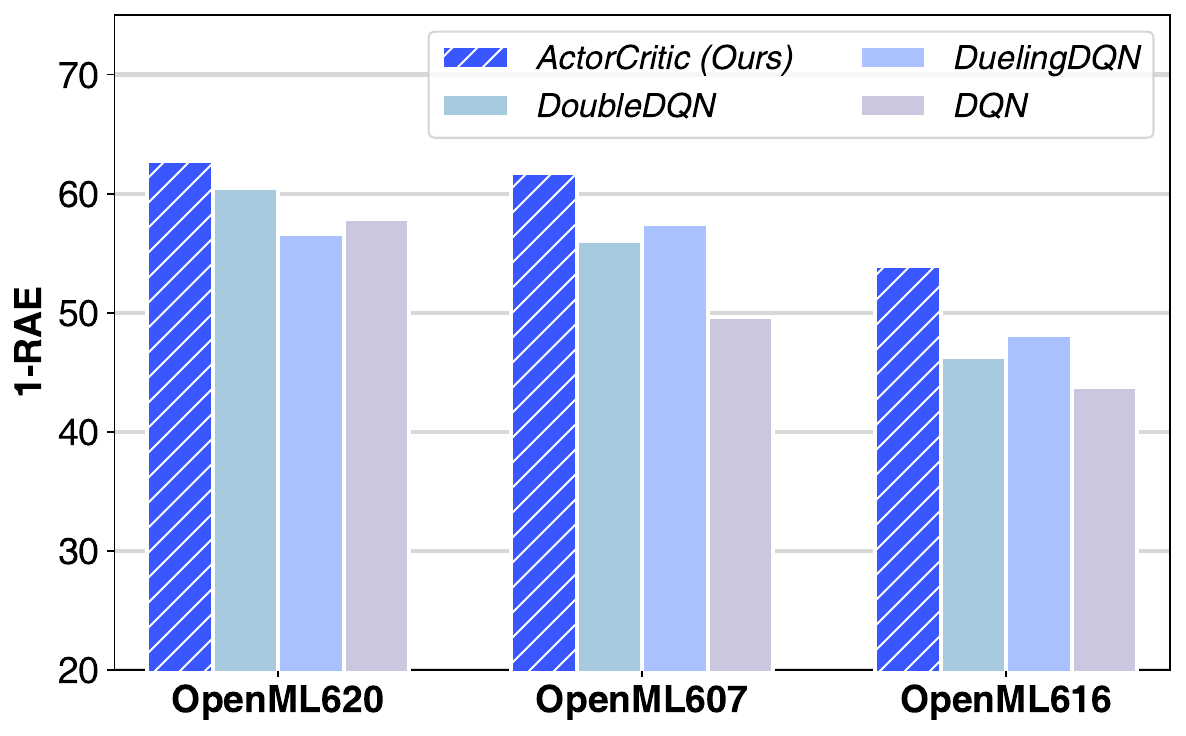}
        \caption{Regression Task}
    \end{subfigure}
    \caption{The impact of different policy model.}
    \vspace{-0.4cm}
    \label{fig:exp_policy_model}
\end{figure*}

\subsection{Analysis of Reinforcement Learning Backend Selection}
\label{section:policy_model}
This experiment aims to answer the question: \textit{Whether the policy model used in \model\ is superior to other existing models?}
We compared ActorCritic~\cite{haarnoja2018soft} (Ours) with DQN~\cite{dqn}, Double DQN~\cite{double_dqn} and Dueling DQN~\cite{dueling_dqn} on binary classification, multi classification and regression tasks in Figure~\ref{fig:exp_policy_model}.
We find that ActorCritic outperforms other methods in all tasks.
This implies that, compared to methods such as DQN, which rely on indirect updating of Q-values, ActorCritic combines policy gradient and value function estimation and is able to directly optimize the policy, reducing variance and improving learning efficiency, resulting in better performance on all tasks.
In conclusion, ActorCritic is a suitable choice for the policy model due to its higher efficiency and stability of strategy optimization.

\begin{table}[htbp]
\centering
\caption{Performance under different shuffling ratios of feature descriptions.
“\textbf{Original}” denotes the performance when using all original features without any feature selection.}
\label{table:suffling_test}
\begin{tabular}{lcccccc}
\toprule
\textbf{Dataset} &\textbf{HRLFS (0\%)} & \textbf{25\%} &\textbf{ 50\%} & \textbf{75\%} & \textbf{100\%} & \textbf{Original} \\
\midrule
German\_Credit & 73.07 & 72.95 & 72.18 & 71.41 & 70.13 & 65.44 \\
Credit\_Default & 80.56 & 79.31 & 78.11 & 76.54 & 75.02 & 76.10\\
SpectF & 87.56 & 87.05 & 86.47 & 85.57 & 84.13 & 75.96\\
SpamBase&92.60&92.24&91.76&90.96&89.91 & 90.13\\
\bottomrule
\end{tabular}
\end{table}

\subsection{Analysis of Misleading Metadata}
\label{sec:misleading_metadata}

This experiment aims to answer the question: \textit{
Whether hybrid feature state extraction is robust to noisy metadata?}
We selected four datasets and randomly shuffled 25\%, 50\%, 75\%, and 100\% of the feature descriptions to test robustness in the face of noisy and permuted metadata, shown in Table \ref{table:suffling_test}. 
We found that the method's performance was not significantly affected by erroneous features across all datasets. 
When 0–50\% of the feature descriptions were scrambled, the method's performance declined by a maximum of 2.45. 
Even when 100\% of the feature descriptions were shuffled, performance declined by only 5.54\%. 
In this case, most datasets with HRLFS still outperformed the “Original” setting (no feature selection) across all datasets, demonstrating both robustness and effectiveness.
This demonstrates that our hybrid feature state extraction component combines semantic and statistical information about features, resulting in exceptional robustness.
Even when the semantic information is inaccurate, the statistical information remains valid, and the method can efficiently explore feature subspaces, maintaining stable performance.
In summary, our method combines semantic information and statistical features, resulting in its stable performance in the face of misleading metadata.

\subsection{Analysis of LLM-Generated Metadata Confidence}
\label{sec:generated_metadata_confidence}

This experiment aims to answer the question: \textit{
Whether LLM-generated metadata is reliable?}
We adopt the LLM credibility evaluation method SelfCheckGPT~\cite{selfcheckgpt} to evaluate the reliability of dataset metadata generated on three datasets with missing values (include two financial datasets German\_Credit and Credit\_default) .
SelfCheckGPT is an LLM confidence metric that asks a large language model to indicate whether a target sentence is supported by the provided samples, thereby providing a confidence score. 
It ranges from 0 to 100, and a lower score indicates a lower level of hallucination in generated content which means a higher level of confidence.
The experimental results are shown in Table \ref{table:selfcheckgpt_score}.
We found that the confidence levels of all three datasets were relatively high. 
This suggests that LLM was able to provide reasonable feature descriptions based on the available information by constructing detailed prompts using dataset metadata and feature names.
Therefore, due to our detailed and specific prompt, generated feature descriptions are credible and virtually free of hallucination issues.

\begin{table}[htbp]
\centering
\caption{SelfCheckGPT score in three datasets}
\label{table:selfcheckgpt_score}
\begin{tabular}{lccccc}
\toprule
\textbf{Dataset} &German\_Credit & Credit\_Default & SPECTF\\
\midrule
\textbf{SelfCheckGPT Score} & 4.16 & 12.31 & 6.81  \\
\bottomrule
\end{tabular}
\end{table}

\subsection{Analysis of Component-Wise Runtime}
\label{sec:component_wise_runtime}
This experiment aims to answer the question: \textit{
How does the time consumption of each component of HRLFS change with the growth of feature dimensions?}
We recorded the time consumption of each component across four datasets with different feature dimensions, shown in Figure~\ref{figure:component_time_consumption}.
We found that Exploration and Optimization and Metadata Embedding dominate the overall runtime of \model, and their proportion of time cost rises further as feature dimension increases.
This implies that GMM, as a statistical approach, can extract feature distributions efficiently and precisely without being sensitive to dimensional growth. And the pairwise merging strategy adopted for constructing hierarchical agents carries an $O(\log N)$ time complexity, enabling efficient execution even under high feature dimensions. Metadata embedding scales linearly with feature dimensions and requires LLM API invocations, which leads to rapid growth in its time overhead. Moreover, the Exploration and Optimization stage involves multi-agent hierarchical decision-making and reward Evaluation via downstream tasks, making it the primary time-consuming module. 
In summary, although Exploration and Optimization and Metadata Embedding account for most of the computational cost, they dominate the performance of \model. Analysis in Section~\ref{sec:space_time_complexity_comparison} also proves that our method retains significant runtime advantages over other baselines. 

\begin{figure}
    \centering
    \includegraphics[width=0.65\linewidth]{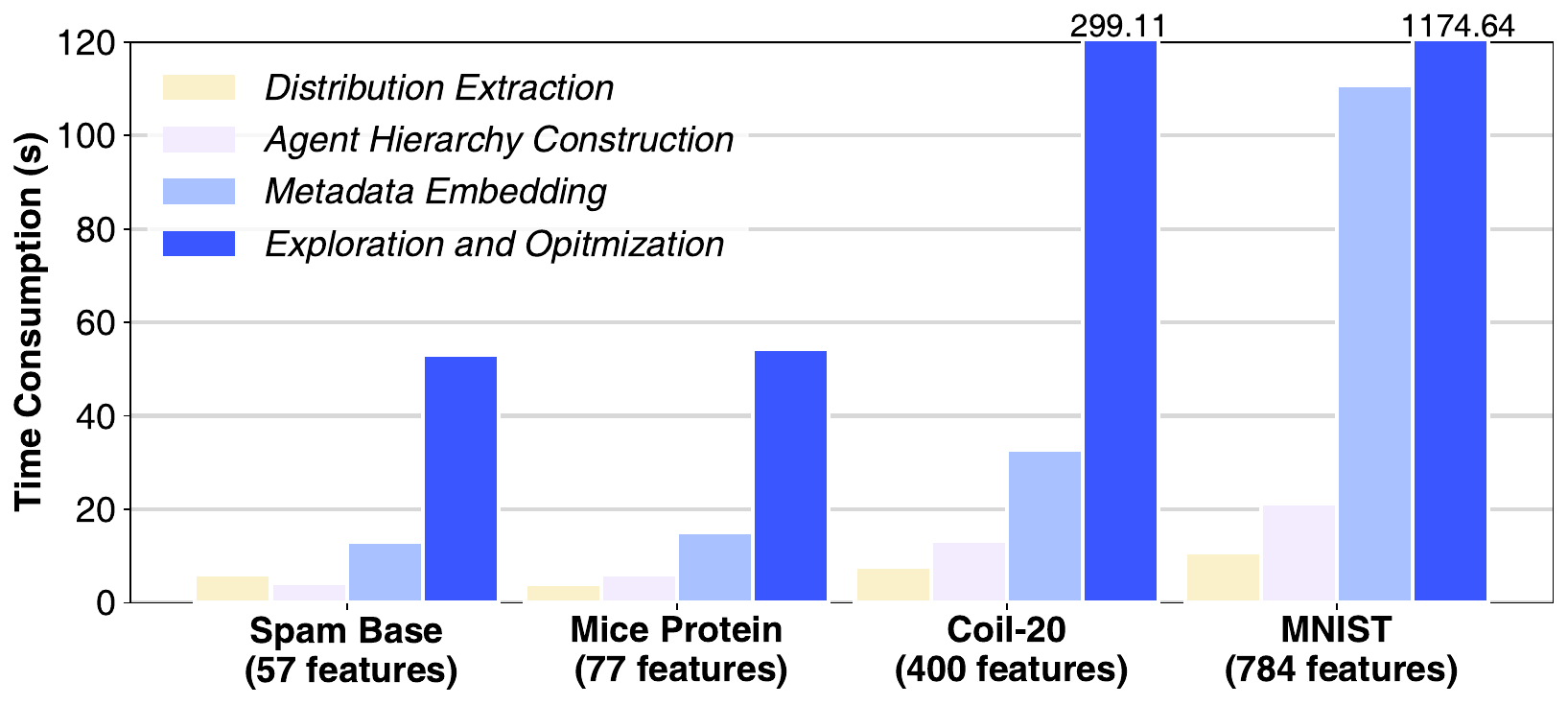}
    \caption{The impact of component time consumption }
    \label{figure:component_time_consumption}
\end{figure}

\subsection{Hyperparameter Study}
\label{hyper_exp}

\begin{figure}[!ht]
\centering 
\begin{subfigure}{0.5\textwidth}
\includegraphics[width=\textwidth]{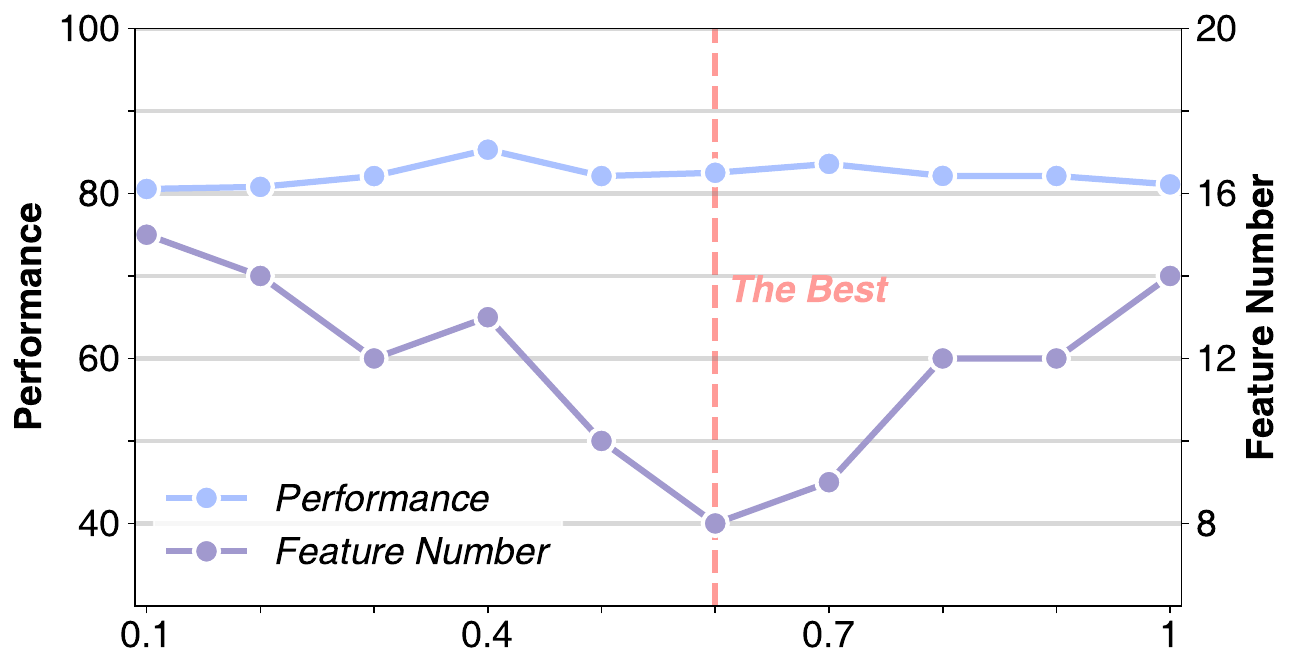}
\caption{$\alpha$ keeps 0.5, $\lambda$ varies from 0.1 to 1.}
\end{subfigure}
\begin{subfigure}{0.48\textwidth}
\includegraphics[width=\textwidth]{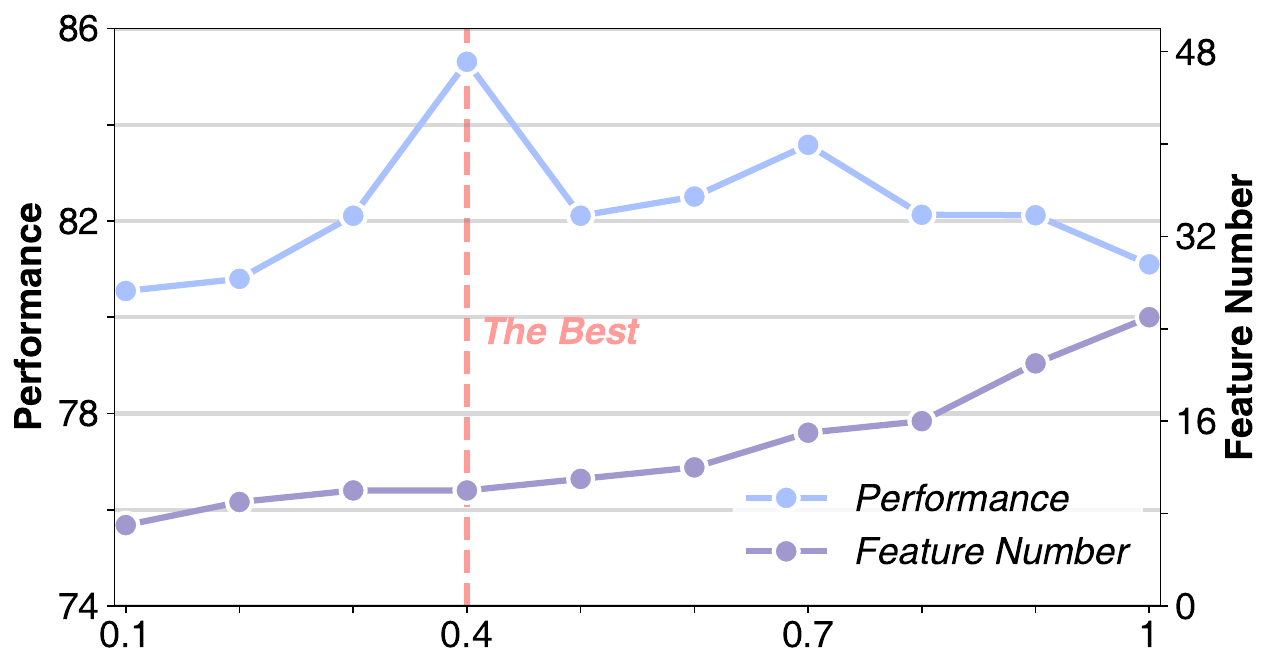}
\caption{$\lambda$ keeps 0.5, $\alpha$ varies from 0.1 to 1.}
\end{subfigure}
\caption{The impact of hyperparameter $\lambda$ and $\alpha$.}
\vspace{-0.4cm}
\label{fig:hyper}
\end{figure}

In this section, we analysis two essential hyperparameter of \model{}, i.e., $\lambda$ (Equation~\ref{equation:reward_lambda}) and $\alpha$ (Equation~\ref{equation:reward_alpha}). 
A higher $\lambda$ encourages the selection of fewer features, while an increase $\alpha$ prioritizes performance over feature compactness. 
We performed experiments adjusting $\lambda$ and $\alpha$ within the range of 0.1 to 1.0 on the SpectF dataset and report the results in Figure~\ref{fig:hyper}.
Analysis reveals that changes in $\lambda$ have minimal impact on model performance but influence the number of selected features, which initially decreases and then increases as $\lambda$ rises. 
This pattern indicates that higher $\lambda$ first suppresses selected feature size, but its further increases reduce fluctuations in $r^{q}_t$, allowing more features to be included. 
Regarding $\alpha$, increasing its value initially enhances model performance by emphasizing performance objectives. 
However, beyond a certain point, performance declines as excessive feature selection introduces redundancy. 
The number of selected features rises with $\alpha$, reflecting the decrease in suppression in the reward function.
These findings demonstrate that $\lambda$ and $\alpha$ are crucial in balancing feature selection and model performance.
Consequently, we selected $\lambda = 0.6$ and $\alpha = 0.4$ to balance performance and feature compactness. 
The experimental results confirm that $\lambda$ and $\alpha$ modulate the reward function in alignment with our objectives.

\subsection{Empirical Validation of the Hierarchical Assumption}\label{sec:empirical-hierarchy}

\begin{figure*}[!h]
\centering
\begin{subfigure}{0.48\textwidth}
\includegraphics[width=\textwidth]{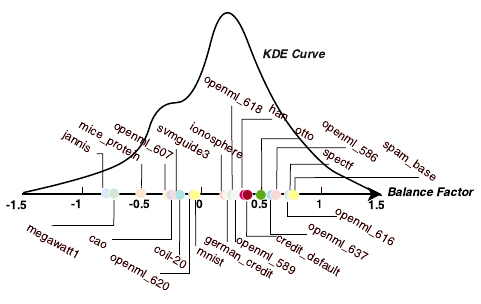}
\caption{Average Node Balance Factor.}
\end{subfigure}
\begin{subfigure}{0.40\textwidth}
\includegraphics[width=\textwidth]{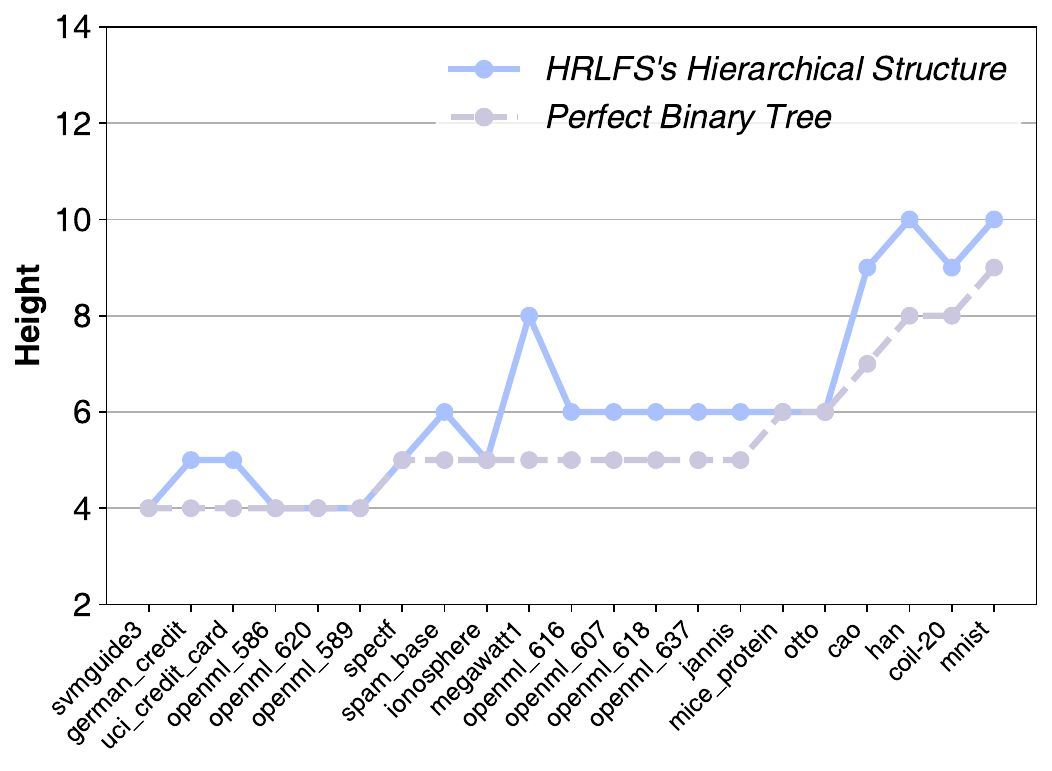}
\caption{Total Tree Height.}
\end{subfigure}
\caption{Empirical study of \model's hierarchical structure compared with the perfect binary tree.}
\label{figure:tree_balance_validate}
\end{figure*}

Our theoretical analysis is motivated by the observation that the constructed hierarchy tends to approximate a perfect binary tree. 
To assess this assumption empirically, we evaluated the balance factor and height of the learned hierarchies across all 21 datasets.
The average node balance factor is defined as the average difference in height between the left and right subtrees for all nodes~\cite{larsen1994avl}. 
Specifically, For an internal node $N$ with left/right subtree heights $h_L(N),h_R(N)$, define the node-wise balance factor
$\beta(N)=|h_L(N)-h_R(N)|$, and the average node balance factor can be defined as: 
$\bar\beta := \frac{1}{|\mathcal{N}_{\mathrm{int}}|}\sum{N\in\mathcal{N}_{\mathrm{int}}}\beta(N)$,
where $\mathcal{N}_{\mathrm{int}}$ is the set of internal nodes. 
In a perfect binary tree, this difference never exceeds 1, and the tree height is uniquely determined by the number of nodes.
As shown in Figure~\ref{figure:tree_balance_validate}, the absolute balance factor remains below 1 for all datasets, and over 80\% of them fall within 0.5. In addition, for 18 datasets, the height of our constructed hierarchy differs from that of the corresponding perfect binary tree by no more than 2 levels. This strong similarity arises because our Agent Hierarchy Construction algorithm clusters agent sets with comparable sizes at earlier stages, naturally promoting balanced splits.
In summary, both balance and height analyses confirm that the constructed hierarchies empirically resemble perfect binary trees, thereby supporting the validity of our assumption.

\section{Related Work}

\smallskip
\noindent\textbf{Feature Selection.}
Feature selection methods are broadly categorized into three approaches: filter-based, wrapper-based, and embedded-based methods~\cite{li2017feature}.
Filter-based methods~\cite{biesiada2008feature,yu2003feature} evaluate the statistical relationship between each feature and the prediction target or among features themselves. Features with measurements below a certain threshold are excluded.
These methods are computationally efficient but do not account for interactions between features, which can result in redundant or suboptimal feature subsets.
Wrapper-based methods\cite{huang2018feature,chen2017kernel,stein2005decision,gheyas2010feature} assess feature subsets by iteratively training and evaluating a specific machine learning model.
However, as the number of features grows, the search space becomes increasingly large, rendering these methods impractical for large-scale datasets~\cite{fan2020autofs,xiao2023traceable}.
Embedded methods\cite{dinh2020consistent,li2016deep,lemhadri2021lassonet} integrate feature selection directly into the model training process. 
This approach allows for the simultaneous optimization of feature relevance and model performance, facilitating the efficient identification of relevant features\cite{lu2018deeppink,koyama2022effective,kumagai2022few}.
However, embedded methods are often tied to specific types of models, limiting their flexibility and transferability to different downstream tasks.
Additionally, some embedded techniques can still be computationally demanding when applied to large and complex datasets.

\smallskip
\noindent\textbf{Reinforcement Learning.}
Reinforcement Learning is a machine learning paradigm in which an agent learns to make decisions by interacting with an environment based on its policy~\cite{sutton2018reinforcement}. RL has been successfully applied to various research areas~\cite{sarkar2023partially,openai2024gpt4technicalreport,zhang2024large,he2025fastft}. 
However, classical RL algorithms often struggle with long-range decision-making tasks due to limitations in scalability and efficiency~\cite{pateria2021hierarchical}.
To overcome these challenges, Hierarchical Reinforcement Learning (HRL) has been developed to decomposes the decision-making process into multiple levels of abstraction, allowing for more efficient learning and planning in complex tasks~\cite{hengst2011hierarchical,WangWYKP24}. 
Building on HRL, Multi-Agent Hierarchical Reinforcement Learning focuses on the coordination and collaboration among multiple HRL agents~\cite{pateria2021hierarchical,liu2016learning,ghavamzadeh2006hierarchical,SivagnanamPLMDL24,MaHLX24}. 
For example, Chakravorty et al.\cite{chakravorty2019option} investigated the learning of temporal abstractions in cooperative multi-agent systems, while Yang et al.\cite{yang2020hierarchical} demonstrated the integration of cooperative MARL algorithms for training high-level policies alongside single-agent RL for low-level skill acquisition. 
These advancements in MAHRL inspire our approach, where we employ a divide-and-conquer strategy within a multi-agent hierarchical reinforcement learning framework.

\smallskip
\noindent\textbf{RL-based Feature Selection.}
Reinforcement learning has been widely employed to recast wrapper-based feature selection as an objective-directed sequential decision problem, enabling iterative subspace exploration.
In the early stage, Liu et al.~\cite{sarlfs} proposed a single-agent-all-feature framework; however, it suffers from poor scalability and high time complexity ($\mathcal{O}(N)$ on high-dimensional data). 
Liu et al.~\cite{marlfs} introduced a multi-agent reinforcement learning (MARL) architectures that assign one agent per feature, but the resulting quadratic growth in agent count quickly becomes prohibitive. 
Group-based and interaction-wise agents\cite{fan2020autofs,fan2021autogfs} reduce the decision horizon by clustering features with simple statistical criteria, yet they ignore semantic and contextual relationships, leading to redundant or sub-optimal selections~\cite{li2024exploring,moraffah2024causal}.
Inspired by hierarchical RL’s success in long-horizon tasks~\cite{hengst2011hierarchical,WangWYKP24}, we adopt a divide-and-conquer multi-agent hierarchical RL framework that first semantically clusters features and then assigns cooperative hierarchical agents to each cluster, simultaneously improving scalability, efficiency, and downstream performance.

\section{Conclusion and Remark}
This paper studies the challenges of exploring feature subspaces from complex datasets.
Specifically, we reformulated the problem of feature selection as a multi-agent hierarchical reinforcement learning framework through the hybrid feature state extraction (comprehend), the hierarchical agent initialization (divide), and the exploration and optimization (conquer). 
We carried out comprehensive evaluations of \model{}, illustrating its efficient, superior, and robust performance in various data sets from different tasks and fields.

\section{Limitations and Discussions}
\model{} relies on metadata and feature names to generate feature descriptions, which makes it vulnerable to adversarial or maliciously crafted inputs that could corrupt the generated descriptions.
Moreover, the current design and evaluation are limited to tabular data. Extending the framework to high-dimensional modalities such as time series, images, or audio remains an open challenge.
Finally, the present formulation focuses on standard supervised learning tasks (classification and regression). Adapting \model{} to more diverse downstream objectives—including survival analysis, anomaly detection, and causal inference—represents a promising direction for future research.

To facilitate the better use of \model{}, we have prepared industrialization notes on the following three aspects.
(1) Cold-start for sparse metadata.
For sparse metadata (e.g., unannotated synthetic datasets), \model{} sets the semantic vector to zero and falls back to distribution-only states ($\theta_i^*$, optimized GMM parameters) when semantic embeddings $e_i$ are unavailable. This preserves core clustering and agent decision logic, maintaining feature selection performance via mathematical feature characteristics alone.
(2) Monitoring selected-feature drift.
We track two drift metrics for \model{} feature clusters: (i) statistical drift of GMM parameters $\theta_i^*$ ($\mu$, $\sigma$), (ii) task performance drift on holdout data. Drift triggers initiate re-exploration of the feature subspace to retain optimality for non-stationary distributions.
(3) Cost envelope for LLM embedding at scale.
For large-scale datasets, LLM embedding costs are controlled via lightweight open-source models (e.g., Qwen3-Embedding-0.6B~\cite{qwen3embedding}) for on-premises deployment. For $>10^5$ features, feature sampling for embedding balances accuracy and computational cost.
(4) Retraining cadence vs feature churn.
\model{} retraining cadence aligns with feature churn rate: (i) $<5\%$ churn: lightweight policy retraining; (ii) $5\%-20\%$: full feature re-clustering + policy fine-tuning; (iii) $>20\%$: complete \model{} retraining. Static datasets only require retraining for task objective/metric changes.

\begin{acks}
This work is partially supported by the National Natural Science Foundation of China (No.62506351, 92470204), the Strategic Priority Research Program of the Chinese Academy of Sciences (No.XDB1350102), and the Beijing Natural Science Foundation (No.4254089). 
\end{acks}

\balance
\normalem

\bibliography{ref}
\bibliographystyle{IEEEtran}

\end{document}